\documentclass[11pt]{article}
\usepackage[margin=1in]{geometry}
\usepackage[T1]{fontenc}
\usepackage[utf8]{inputenc}
\usepackage{lmodern}
\usepackage[hyphens]{url}
\usepackage{graphicx}
\usepackage{natbib}
\usepackage{caption}
\usepackage{hyperref}
\hypersetup{
    colorlinks=true,
    citecolor=blue,
    linkcolor=red,
    urlcolor=blue
}
\usepackage[symbol]{footmisc}

\usepackage{booktabs}       
\usepackage{amsfonts}       
\usepackage{nicefrac}       
\usepackage{microtype}      
\usepackage{xcolor}         

\title{Efficient Online Lexicographic Generalized Low-Rank Matrix Bandits}

\author{
  Bo Xue\thanks{City University of Hong Kong.
  Email: \texttt{boxue4-c@my.cityu.edu.hk}}
  \quad
  Ji Cheng\thanks{City University of Hong Kong.
  Email: \texttt{J.Cheng@my.cityu.edu.hk}}
  \quad
  Haodong Jing\thanks{Xi'an Jiaotong University.
  Email: \texttt{jinghd@stu.xjtu.edu.cn}}
  \quad
  Hongzong Li \thanks{Northwestern Polytechnical University.
  Email: \texttt{lihongzong@nwpu.edu.cn}}
  \quad
  Shuang Qiu{\color{blue}${}^\dag$}\thanks{City University of Hong Kong.
  Email: \texttt{shuanqiu@cityu.edu.hk}}
}
    \makeatletter
\def\@fnsymbol#1{\ensuremath{\ifcase#1\or *\or \dagger\or \ddagger\or
   \mathsection\or \mathparagraph\or \natural\or \sharp \or\dagger\dagger
   \or \ddagger\ddagger \else\@ctrerr\fi}}
    \makeatother
    
\usepackage{paralist,amssymb,mathtools,amsmath}
\usepackage{multirow}
\usepackage{bm}
\usepackage{algorithm,algorithmic}
\usepackage{color}

\newtheorem{thm}{Theorem}

\newtheorem{lem}{Lemma}

\newtheorem{rem}{Remark}

\newtheorem{assume}{Assumption}

\def \E {\mathrm{E}}
\def \Pr {\mathrm{Pr}}

\def \A {\mathrm{A}}
\def \B {\mathrm{B}}
\def \X {\mathrm{X}}
\def \U {\mathrm{U}}
\def \DM {\mathrm{D}}
\def \U {\mathrm{U}}
\def \V {\mathrm{V}}
\def \I {\mathrm{I}}

\def \R {\mathbb{R}}

\def \D {\mathcal{D}}

\def \SX {\mathcal{X}}

\usepackage{titling}
\thanksmarkseries{arabic}

\DeclareMathOperator*{\argmin}{argmin}
\DeclareMathOperator*{\argmax}{argmax}

\DeclarePairedDelimiter\norm{\lVert}{\rVert}

\begin{document}

\date{\today}

\footnotetext[2]{Corresponding Author}

\maketitle

\begin{abstract}
This paper studies generalized low-rank matrix bandits with multiple prioritized objectives. At each round, the learner selects a matrix-valued arm and observes a vector-valued reward, whose components correspond to multiple objectives with different priority levels. Each objective is governed by an objective-specific generalized low-rank matrix model, and the learner evaluates arms according to a lexicographic preference order, prioritizing higher-level objectives before lower-level ones. We propose \textsc{Lexi-LowGLM}, an efficient online algorithm that first estimates objective-specific low-rank subspaces and then performs lexicographic learning in the reduced feature spaces. Unlike existing single-objective algorithms that repeatedly solve a batch generalized linear estimator using all historical observations, \textsc{Lexi-LowGLM} updates each objective-specific estimator via an online Newton step, reducing the estimator-update complexity over $T$ rounds from $O(T^2)$ to $O(T)$. We establish a regret bound of $\widetilde O\left(W_i^{\rm lex}\sqrt{m}\,(d_1+d_2)r\sqrt{T}\right)$ for each objective $i\in[m]$, where $r$ is an upper bound on the ranks of the objective-specific parameter matrices and $W_i^{\rm lex}$ characterizes the lexicographic trade-off effect. This bound depends on the effective low-rank dimension $(d_1+d_2)r$ rather than the ambient dimension $d_1d_2$. Numerical experiments further validate the effectiveness and computational efficiency of the proposed method.
\end{abstract}

\section{Introduction}
Contextual bandits provide a fundamental framework for sequential decision-making under uncertainty \citep{Robbins:1952,Lai:1985,Auer:2002}, with widespread applications in personalized recommendation \citep{Li:2010}, online advertising \citep{Schwartz:2017}, and resource allocation \citep{Khansa:2021}. At each round, the learner observes a context, selects an action, and receives feedback only for the chosen action. In many applications, however, actions are more naturally represented as matrices rather than vectors, especially when rewards depend on pairwise interactions between two feature groups, such as users and items in recommendation, flights and hotels in bundle selection, or agents and tasks in matching platforms \citep{Jun:2019,Kang:2022}. Although matrix-valued actions can be vectorized and handled using standard bandit algorithms \citep{Dani:2008,Filippi:2010}, such a reduction discards their intrinsic structural information and may lead to substantial statistical and computational inefficiency in high-dimensional problems.

To address this challenge, generalized low-rank matrix bandits assume that the unknown parameter matrix has low rank \citep{Jun:2019}. Specifically, the expected reward of an arm $\X\in\R^{d_1\times d_2}$ is modeled as $\mu(\langle \X,\Theta^*\rangle)$, where $\mu(\cdot)$ is an inverse link function and $\Theta^*$ is an unknown rank $r$ matrix. By exploiting the row and column subspaces of $\Theta^*$, existing methods achieve regret bounds that scale with the effective low-rank dimension, roughly $(d_1+d_2)r$, rather than the ambient dimension $d_1d_2$ \citep{Jang:2021,Lu:2021-matrix}. However, existing generalized low-rank matrix bandit algorithms are mainly developed for scalar rewards \citep{Kang:2024,Wang:2025}, which limits their applicability to online decision-making problems involving multiple objectives.

In many applications, the learner observes vector-valued feedback that captures several objectives to be optimized jointly, and these objectives often have different priorities. For example, radiation treatment planning jointly considers target coverage and the protection of organs at risk, with target coverage typically assigned higher priority \citep{Jee:2007}. Similarly, water resource planning involves competing objectives such as flood protection, irrigation shortage reduction, and electricity generation, whose importance is naturally ordered by practical needs \citep{Weber:2002}. These examples motivate the use of lexicographic preferences \citep{Matthias:2005}, under which higher-priority objectives are optimized first and lower-priority objectives are considered only among arms that remain competitive with respect to all higher-priority objectives.


Although lexicographic bandits have been studied in various settings~\citep{Tekin:2018,Huyukt:2021,AAAI:2025:Xue}, existing methods do not account for generalized low-rank matrix structures. Directly applying them to vectorized matrix arms would cause both regret and computational complexity to scale with the ambient dimension $d_1d_2$, thereby losing the advantages of low-rank modeling. Conversely, existing generalized low-rank matrix bandit algorithms primarily focus on scalar rewards~\citep{Kang:2022} and therefore cannot directly accommodate multiple objectives with strict priorities. 

Beyond the challenge of jointly exploiting low-rank structures and coordinating decisions according to lexicographic priorities, computational efficiency poses an additional challenge. Existing generalized low-rank matrix bandit methods typically recompute a batch generalized linear estimator using all historical observations at every round~\citep{Kang:2022}. Although this approach facilitates theoretical analysis, it repeatedly processes past data and incurs a cumulative estimator-update complexity of $O(T^2)$ over $T$ rounds, making it unsuitable for long-horizon online learning. Therefore, a natural question arises:
\begin{center}
\textit{Can we design a computationally efficient algorithm for lexicographic generalized low-rank matrix bandits?}
\end{center}

In this paper, we study \emph{lexicographic generalized low-rank matrix bandits}, an online learning problem with matrix-valued arms and multiple prioritized objectives. Our main contributions are summarized as follows.
\begin{itemize}
    \item \textbf{Problem formulation.} To the best of our knowledge, we are the first to formulate generalized low-rank matrix bandits with multi-objective feedback, moving beyond the conventional scalar-reward matrix bandits to settings in which multiple objectives are optimized simultaneously.

    \item \textbf{Efficient algorithm.} We propose \textsc{Lexi-LowGLM}, an efficient lexicographic learning algorithm that exploits objective-specific low-rank structure while updating estimators online. By replacing repeated batch estimation with an online Newton-type proximal update, \textsc{Lexi-LowGLM} reduces the cumulative estimator-update complexity over $T$ rounds from $O(T^2)$ to $O(T)$.

    \item \textbf{Regret guarantee.} We establish an objective-wise regret bound for \textsc{Lexi-LowGLM}. Specifically, for each objective $i\in[m]$, the regret scales as $\widetilde O(W_i^{\rm lex}\sqrt{m}\,(d_1+d_2)r\sqrt{T})$, where $W_i^{\rm lex}=1+w+\cdots+w^{i-1}$ quantifies the effect of lexicographic priority propagation. This result matches the state-of-the-art dependence on the horizon and the effective low-rank dimension in the single-objective setting \citep{Kang:2022}, while simultaneously controlling the regret of all objectives. When $w=0$, we have $W_i^{\rm lex}=1$ for all $i\in[m]$, so the lexicographic structure incurs no additional priority-propagation penalty, yielding a uniform regret guarantee across all objectives.

    \item \textbf{Empirical validation.} We conduct numerical experiments to verify the effectiveness and computational efficiency of the proposed algorithm.
\end{itemize}

\section{Related Work}
In this section, we review related work on matrix bandits and multi-objective bandits.

\paragraph{Matrix Bandits.} Matrix bandits extend contextual linear bandits \citep{Abbasi:2011,Chu:2011} to matrix-valued arms by exploiting structural assumptions on the unknown parameter matrix. A representative line of work focuses on low-rank structure, where the goal is to avoid the ambient-dimensional dependence incurred by vectorizing matrix arms. \citet{Jun:2019} introduced bilinear bandits and proposed an explore-subspace-then-refine strategy. Subsequent works improved this line by leveraging the geometry of the action space \citep{Jang:2021}, studying pure exploration with shared representations \citep{Mukherjee:2023}, and developing robust algorithms under heavy-tailed rewards \citep{Kang:2024}.

More recent works consider generalized low-rank matrix bandits, where the expected reward follows a generalized linear model. \citet{Lu:2021-matrix} studied this problem beyond the bilinear setting, but their covering-based algorithm can be computationally expensive. \citet{Kang:2022} improved computational tractability through Stein-type subspace estimation and low-rank generalized linear bandit learning. Other extensions further exploit arm-set geometry \citep{Jang:2024}, graph information \citep{Wang:2025}, or low-rank structures in related feedback models \citep{Lee:2026}. However, these works focus on scalar rewards and do not address prioritized multi-objective feedback, which is the focus of this paper. 

\paragraph{Multi-Objective Bandits.} Multi-objective bandits study online decision-making with vector-valued rewards. A common approach is to aggregate multiple objectives into a scalar reward \citep{Drugan:2013}. \citet{Busa-Fekete:2017} optimized the generalized Gini index in multi-objective bandits, while more recent work studied nonlinear scalarizations such as hypervolume scalarization and established sublinear hypervolume regret for multi-objective stochastic linear bandits \citep{Zhang:2024}. Another line of work adopts Pareto optimality as the preference model. Early studies considered Pareto regret in multi-objective Multi-Armed Bandits (MABs) and contextual bandits \citep{Turgay:2018,Lu:2019,Cai:2023}, and recent work further developed Pareto regret analysis for both stochastic and adversarial multi-objective MABs \citep{Xu:2023,Park:2025}. Related pure-exploration studies aim to identify the Pareto front with sample-complexity guarantees \citep{Auer:2016,Crepon:2024}. Moreover, recent studies on multi-objective reinforcement learning with bandit feedback \citep{qiu2024traversing}  further extend multi-objective MABs to the multi-objective Markov decision processes.

Another important formulation for multi-objective decision-making is based on lexicographic preferences, which explicitly encode priority relations among objectives. Under lexicographic ordering, objectives are optimized sequentially according to their importance, so that lower-priority objectives are considered only after higher-priority objectives have been sufficiently addressed. Lexicographic preferences have been studied in multi-objective MABs \citep{Huyukt:2021}, contextual bandits \citep{Tekin:2018}, and stochastic linear bandits \citep{AAAI:2025:Xue}. Related lexicographic preference models have also been explored in reinforcement learning (RL), including lexicographic multi-objective RL \citep{Skalse:2022,Tercan:2024} and lexicographic linear MDPs \citep{ICML:2025:Xue}. However, these methods do not consider low-rank matrix structures.

\section{Preliminaries}\label{sec:problem_setting}
In this section, we formulate the lexicographic generalized low-rank matrix bandit problem and introduce several auxiliary tools for handling the low-rank structure.

\paragraph{Notations.}
Let $[n]=\{1,2,\ldots,n\}$ for any positive integer $n$. Given a vector $x\in\R^d$, $\|x\|_2$ denotes its Euclidean norm, and for a positive definite matrix $\V\in\R^{d\times d}$, its weighted norm is $\|x\|_{\V}=\sqrt{x^\top \V x}$. For a matrix $\A\in\R^{d_1\times d_2}$, $\|\A\|_F$, $\|\A\|_{\mathrm{op}}$, and $\|\A\|_{\mathrm{nuc}}$ denote its Frobenius norm, operator norm, and nuclear norm, respectively. The inner product between two matrices $\A,\B\in\R^{d_1\times d_2}$ is defined as $\langle \A,\B\rangle=\operatorname{trace}(\A^\top \B)$. $\operatorname{vec}(\A)$ denotes the vectorization of $\A$ obtained by stacking its columns, and $\widetilde O(\cdot)$ hides logarithmic factors.

\paragraph{Learning Model.} Lexicographic generalized low-rank matrix bandits consider a $T$-round sequential decision-making problem. At each round $t\in[T]$, the learner first observes the arm set $\SX\subseteq\R^{d_1\times d_2}$, selects an arm $\X_t\in\SX$, and then receives a stochastic reward vector $y_t=(y_{t,1},\ldots,y_{t,m})\in\R^m$, where $m$ is the number of objectives. For each objective $i\in[m]$, the expected reward is modeled by an objective-specific generalized low-rank matrix model:
\begin{equation*}
\E[y_{t,i}\mid \X_t]=\mu_i\!\left(\left\langle \X_t,\Theta_i^*\right\rangle\right),
\end{equation*}
where $\mu_i(\cdot)$ is an inverse link function, and $\Theta_i^*\in\R^{d_1\times d_2}$ is an unknown low-rank parameter matrix. Equivalently, the observed reward can be written as
\begin{equation*}
y_{t,i}=\mu_i\!\left(\left\langle \X_t,\Theta_i^*\right\rangle\right)+\eta_{t,i}, i\in[m],
\end{equation*}
where $\eta_{t,i}$ is zero-mean noise.

\paragraph{Lexicographic Preference and Regret.} For any two arms $\X,\X'\in\SX$, we say that $\X$ \emph{lexicographically dominates} $\X'$ if there exists an objective index $i\in[m]$ such that
\begin{equation*}
\begin{aligned}
\mu_j\!\left(\left\langle \X,\Theta_j^*\right\rangle\right)&=\mu_j\!\left(\left\langle \X',\Theta_j^*\right\rangle\right), \forall j\in[i-1], \\
\mu_i\!\left(\left\langle \X,\Theta_i^*\right\rangle\right)&>\mu_i\!\left(\left\langle \X',\Theta_i^*\right\rangle\right).
\end{aligned}
\end{equation*}
An arm $\X_*\in\SX$ is \emph{lexicographically optimal} if it is not lexicographically dominated by any other arm in $\SX$.

We evaluate the performance of the learner using the objective-wise cumulative regret with respect to the lexicographically optimal arm $\X_*$. Specifically, for each objective $i\in[m]$, the cumulative regret is defined as
\begin{equation*}
R_i(T)=\sum_{t=1}^T
\left[
  \mu_i\!\left(\left\langle \X_*,\Theta_i^*\right\rangle\right)-\mu_i\!\left(\left\langle \X_t,\Theta_i^*\right\rangle\right)
\right].
\end{equation*}
This quantity measures the cumulative loss on objective $i$ incurred by selecting $\X_t$ instead of the lexicographically optimal arm $\X_*$. The goal is to design an algorithm that achieves sublinear growth of $R_i(T)$ for every objective $i\in[m]$, thereby ensuring that the learner approaches the lexicographic optimal arm while controlling the regret of every objective.

We impose the following assumptions, which are standard in the literature on generalized low-rank matrix bandits \citep{Kang:2024,Wang:2025} and lexicographic bandit learning \citep{AAAI:2025:Xue}.

\begin{assume}[Low-rank structure]\label{ass:low_rank}
For each objective $i\in[m]$, the unknown parameter matrix $\Theta_i^*$ is of rank at most $r$, i.e., $\operatorname{rank}(\Theta_i^*)\le r$, where $r\ll \min\{d_1,d_2\}$.
\end{assume}

\begin{assume}[Bounded parameters and arms]\label{ass:bounded_norm}
There exists a constant $S>0$ such that $\|\Theta_i^*\|_F\le S$ for all $i\in[m]$. 
The arm set is normalized so that $\|\X\|_F\le 1$ for all $\X\in\SX$.
\end{assume}

\begin{assume}[Regular link functions]\label{ass:link_regular}
For each objective $i\in[m]$, the inverse link function $\mu_i(\cdot)$ is continuously differentiable. There exist constants $0<c_\mu\le L_\mu$ such that, over the relevant domain, $c_\mu \le \mu'_i(z) \le L_\mu, \forall i\in[m]$.
\end{assume}

\begin{assume}[Bounded rewards and noise]\label{ass:bounded_noise}
There exist constants $U,R>0$ such that $|\mu_i(z)|\le U$ over the relevant domain and $\left|y_{t,i}-\mu_i\!\left(\left\langle \X_t,\Theta_i^*\right\rangle\right)\right|\le R, \forall t\in[T],\ i\in[m]$.
\end{assume}

\begin{assume}[Lexicographic trade-off]\label{ass:lex_tradeoff}
There exists $w\ge 0$ such that, for every $i\in\{2,\ldots,m\}$ and any $\X\in\SX$,
\begin{equation*}
\begin{aligned} 
&\qquad\mu_i\left(\left\langle \X,\Theta_i^*\right\rangle\right)-\mu_i\left(\left\langle \X_*,\Theta_i^*\right\rangle\right)\\
&\leq w \cdot \max_{j\in[i-1]}\left\{\mu_j\left(\left\langle \X_*,\Theta_j^*\right\rangle\right)-\mu_j\left(\left\langle \X,\Theta_j^*\right\rangle\right)\right\}.
\end{aligned}
\end{equation*}
\end{assume}

\paragraph{Auxiliary Tools for Subspace Estimation.} We introduce several auxiliary tools for estimating the low-rank row and column subspaces. These tools follow the Stein-type subspace estimation approach developed for generalized low-rank matrix bandits \citep{Kang:2022}.

Let $p:\R\to\R$ denote a univariate probability density function. Its corresponding score function is defined by
\begin{equation*}
S^p(x)=-\nabla_x\log p(x)=-\frac{\nabla_x p(x)}{p(x)},\quad x\in\R.
\end{equation*}
For a random matrix $\A\in\R^{d_1\times d_2}$ with entrywise density $\mathbf p=(p_{ij}):\R^{d_1\times d_2}\to\R^{d_1\times d_2}$, its score function is
\begin{equation*}
S^{\mathbf p}(\A)=\big(S^{p_{ij}}(\A_{ij})\big)_{i,j}\in\R^{d_1\times d_2}.
\end{equation*}
When the underlying exploration distribution is clear from the context, we simply write $S(\A)$.

\begin{assume}[Exploration distribution]\label{ass:exploration}
There exists a sampling distribution $\D$ over $\SX$ such that, for $\X\sim\D$ with associated entrywise density $\mathbf p$, $\E\!\left[\big(S^{\mathbf p}(\X)\big)_{ij}^{2}\right]\le M, \forall (i,j)\in[d_1]\times[d_2].$ Moreover, either the rows or the columns of $\X$ are pairwise independent.
\end{assume}

We also introduce the Hermitian dilation operator. For any matrix $\A\in\R^{d_1\times d_2}$, define
\begin{equation*}
\mathcal H(\A)=
\begin{pmatrix}
0 & \A\\
\A^\top & 0
\end{pmatrix}
\in\R^{(d_1+d_2)\times(d_1+d_2)}.
\end{equation*}
Moreover, for any real-valued function $g:\R\to\R$ and any symmetric matrix $\A=\U\DM\U^\top$, define
\begin{equation*}
g(\A)=\U\operatorname{diag}\big(g(\DM_{11}),\ldots,g(\DM_{dd})\big)\U^\top.
\end{equation*}

To obtain a robust matrix estimator under the finite-moment condition in Assumption~\ref{ass:exploration}, we use the truncation function
\begin{equation*}
\psi(x)=
\begin{cases}
\log(1+x+x^2/2), & x\geq 0,\\
-\log(1-x+x^2/2), & x<0.
\end{cases}
\end{equation*}
Given $\nu>0$, define the matrix-valued truncation operator
\begin{equation}\label{eq:truncation_operator}
\widetilde\psi_\nu(\A)
=
\frac{1}{\nu}
\left[
\psi\big(\nu\mathcal H(\A)\big)
\right]_{1:d_1,\,(d_1+1):(d_1+d_2)} .
\end{equation}
Specifically, $\widetilde\psi_\nu(\cdot)$ first applies $\psi(\cdot)$ to the Hermitian dilation of a rectangular matrix and then extracts its upper-right block. This operator serves as a key component in constructing robust estimators of the objective-specific low-rank subspaces.

\section{Algorithms}
This section presents our algorithms. We first introduce a robust objective-specific subspace estimation procedure to construct transformed feature representations. Based on this, we develop two algorithms. The first, \textsc{Scalar-LowGLM}, is a straightforward scalarized extension of the single-objective algorithm \citep{Kang:2022}, retaining its batch estimator updates. The second, \textsc{Lexi-LowGLM}, directly exploits lexicographic preferences through sequential candidate filtering and online estimator updates, leading to improved regret guarantees and lower computational complexity.

\subsection{Objective-Specific Subspace Estimation}
We begin by estimating the objective-specific low-rank row and column subspaces. To this end, during the first $T_1$ rounds, the learner samples arms independently from the exploration distribution $\D$ over $\SX$ and observes the corresponding reward vector. These exploration samples are then used to construct a separate low-rank estimator for each objective.

For each objective $i\in[m]$, define the empirical loss
\begin{equation}\label{eq:subspace_estimation_objective}
L_{T_1,i}(\Theta) = \langle \Theta,\Theta\rangle - \frac{2}{T_1}\sum_{t=1}^{T_1}
\left\langle \tilde{\psi}_{\nu}\!\big(y_{t,i}\cdot S(\X_t)\big),\Theta \right\rangle,
\end{equation}
where $S(\X_t)$ is the score function associated with the exploration distribution  $\D$, and $\widetilde{\psi}_{\nu}(\cdot)$ is the matrix-valued truncation operator defined in Eq.~\eqref{eq:truncation_operator}. 

The objective-specific estimator is then obtained by solving the nuclear-norm regularized optimization problem
\begin{equation}\label{eq:subspace_estimation}
\widehat{\Theta}_{i}=\arg\min_{\Theta\in\R^{d_1\times d_2}} \left\{L_{T_1,i}(\Theta)+\lambda_{T_1}\|\Theta\|_{\mathrm{nuc}}\right\},
\end{equation}
where the nuclear norm promotes low-rank structure. We set
\begin{equation*}
\nu=\sqrt{\frac{2\log(2m(d_1+d_2)/\delta)}{(4R^2+U^2)MT_1(d_1+d_2)}}
\end{equation*}
 and 
\begin{equation*}
 \lambda_{T_1}=4\sqrt{\frac{2(4R^2+U^2)M(d_1+d_2)\log(2m(d_1+d_2)/\delta)}{T_1}}.
\end{equation*}
After computing $\widehat{\Theta}_i$, we take its singular value decomposition
\begin{equation*}
\widehat{\Theta}_{i} =
[\widehat \U_i,\widehat \U_{i,\perp}] \widehat \DM_i [\widehat \V_i,\widehat \V_{i,\perp}]^\top,
\end{equation*}
where $\widehat \U_i$ and $\widehat \V_i$ contain the leading $r$ left and right singular vectors, respectively. These matrices provide estimates of the objective-specific row and column subspaces for constructing the transformed feature representations.

\begin{algorithm}[t]
\caption{Objective-Specific Subspace Estimation}
\label{alg:subspace_estimation}
\begin{algorithmic}[1]
\REQUIRE $\mathcal X, T_1, r, \mathcal D,\delta$
\FOR{$t=1,2,\ldots,T_1$}
    \STATE Sample $\X_t\sim\D$ and observe $y_t=(y_{t,1},\ldots,y_{t,m})$
\ENDFOR
\FOR{$i=1,2,\ldots,m$}
    \STATE Define the loss function $L_{T_1,i}(\Theta)$ as in Eq.~\eqref{eq:subspace_estimation_objective}
    \STATE Compute $\widehat{\Theta}_{i}$ by solving Eq.~\eqref{eq:subspace_estimation}
    \STATE Compute the SVD $\widehat{\Theta}_{i}=[\widehat \U_i,\widehat \U_{i,\perp}]\,\widehat \DM_i\,[\widehat \V_i,\widehat \V_{i,\perp}]^\top$
    \STATE Define the transformed feature map $f_i(\cdot)$ as in Eq.~\eqref{eq:reduced_feature_map}
\ENDFOR
\RETURN $\{f_i(\cdot)\}_{i=1}^m$
\end{algorithmic}
\end{algorithm}

\paragraph{Transformed Feature Representation.}
For each objective $i\in[m]$, define the rotation operator
\begin{equation*}
\mathcal R_i(\A):=
[\widehat \U_i,\widehat \U_{i,\perp}]^\top \A [\widehat \V_i,\widehat \V_{i,\perp}].
\end{equation*}
The rotated matrix is then partitioned according to the estimated rank-$r$ row and column subspaces. We define the projection operator
\begin{equation*}
\Pi_r(\A):=
\begin{bmatrix}
\operatorname{vec}(\A_{1:r,\,1:r})\\
\operatorname{vec}(\A_{r+1:d_1,\,1:r})\\
\operatorname{vec}(\A_{1:r,\,r+1:d_2})\\
\operatorname{vec}(\A_{r+1:d_1,\,r+1:d_2})
\end{bmatrix},
\end{equation*}
which vectorizes the four resulting blocks. The transformed feature map for objective $i$ is therefore given by
\begin{equation}\label{eq:reduced_feature_map}
f_i(\X):=\Pi_r\!\big(\mathcal R_i(\X)\big). 
\end{equation}
The first $k=(d_1+d_2)r-r^2$ coordinates correspond to the estimated low-rank subspace, while the remaining coordinates capture the complementary directions. This representation allows the next stage to exploit the estimated low-rank structure through anisotropic regularization.

\subsection{Scalarized Batch Method: \textsc{Scalar-LowGLM}}
Algorithm~\ref{alg:scalarized_batch_lowglm} presents a natural baseline that extends the single-objective low-rank matrix bandit algorithm \citep{Kang:2022} to the multi-objective setting.

After estimating the objective-specific low-rank subspaces using Algorithm~\ref{alg:subspace_estimation}, Algorithm~\ref{alg:scalarized_batch_lowglm} performs bandit learning in the transformed feature spaces $\{f_i(\X):\X\in\SX\}_{i=1}^m$. Let
\begin{equation*}
p=d_1d_2, \quad k=(d_1+d_2)r-r^2.
\end{equation*}
The first $k$ coordinates of $f_i(\X)$ correspond to the estimated low-rank subspace, while the remaining $p-k$ coordinates represent its orthogonal complement. To exploit this structure, we use the anisotropic regularization matrix 
\[
\Lambda = \operatorname{diag}(\lambda_0 \I_k,\lambda_\perp \I_{p-k}),
\]
where $\lambda_\perp\gg\lambda_0$. Thus, directions outside the estimated low-rank subspace are regularized more heavily, encouraging learning to concentrate on the informative low-rank subspace.

For each objective $i\in[m]$, we maintain a design matrix $\V_{t,i}$ and a parameter estimator $\hat{\theta}_{t,i}$ in the transformed feature space. They are initialized as $\V_{1,i}=\Lambda$ and $\hat{\theta}_{1,i}=0$, with confidence radius $\beta_{1}=L_\mu\cdot(\sqrt{\lambda_0}S+\sqrt{\lambda_\perp}S_\perp)$.

At round $t$, the predicted reward and confidence width of arm $\X\in\SX$ under objective $i\in[m]$ are given by
\begin{equation}\label{eq:predicted_reward_and_uncertainty}
\hat{y}_{t,i}(\X)=\mu_i\!\big(f_i(\X)^\top \hat{\theta}_{t,i}\big),\ 
c_{t,i}(\X)=\beta_t\|f_i(\X)\|_{\V_{t,i}^{-1}}.
\end{equation}
Here, $c_{t,i}(\X)$ quantifies the uncertainty associated with the estimated reward. \textsc{Scalar-LowGLM} then forms a scalarized upper confidence bound by aggregating the objective-wise optimistic estimates as follows:
\begin{equation}\label{eq:scalarized_ucb}
\mathrm{UCB}_t(\X,w) = \sum_{i=1}^m (1+w)^{m-i} \left(\hat y_{t,i}(\X)+c_{t,i}(\X)\right).
\end{equation}
The scalarization encodes the priority among objectives by assigning larger weights to higher-priority objectives. The algorithm then plays the arm with the largest scalarized UCB value, i.e., $\X_t=\argmax_{\X\in\SX}\mathrm{UCB}_t(\X,w)$.

\begin{algorithm}[t]
  \caption{\textsc{Scalar-LowGLM}}
  \label{alg:scalarized_batch_lowglm}
  \begin{algorithmic}[1]
  \REQUIRE $T, T_1, \delta, r, w, \lambda_0, \lambda_\perp, S_\perp$
  \STATE Run Algorithm~\ref{alg:subspace_estimation} for $T_1$ rounds and obtain $\{f_i(\cdot)\}_{i=1}^m$
  \STATE Set $p=d_1d_2$ and $k=(d_1+d_2)r-r^2$
  \STATE Set $\Lambda=\mathrm{diag}(\lambda_0 \I_k,\lambda_\perp \I_{p-k})$ 
  \STATE Initialize $\V_{1,i}=\Lambda$ and $\hat{\theta}_{1,i}=0$ for all $i\in[m]$, and set the confidence radius $\beta_{1}=L_\mu\cdot(\sqrt{\lambda_0}S+\sqrt{\lambda_\perp}S_\perp)$
  \FOR{$t=1,\ldots,T-T_1$}
    \STATE Compute $\hat{y}_{t,i}(\X)$ and $c_{t,i}(\X)$ for all $\X\in\SX$ by Eq.~\eqref{eq:predicted_reward_and_uncertainty}
    \STATE Compute $\text{UCB}_t(\X,w)$ for all $\X\in\SX$ by Eq.~\eqref{eq:scalarized_ucb} 
    \STATE Play $\X_t=\argmax_{\X\in\SX} \text{UCB}_t(\X,w)$ 
    \STATE Observe reward vector $y_t=(y_{t,1},y_{t,2},\ldots,y_{t,m})$
    \FOR{$i=1,\ldots,m$}
      \STATE Compute the estimator $\hat{\theta}_{t+1,i}$ by Eq.~\eqref{eq:batch-update-baseline}
      \STATE Update $\V_{t+1,i}=\V_{t,i}+\frac{c_\mu}{2}f_i(\X_t)f_i(\X_t)^\top$
    \ENDFOR
    \STATE Compute $\beta_{t+1}$ by Eq.~\eqref{radius-batch}
  \ENDFOR
\end{algorithmic}
\end{algorithm}

After observing the reward vector $y_t=(y_{t,1},\ldots,y_{t,m})$, the estimator for each objective $i\in[m]$ is updated by solving the following regularized empirical risk minimization problem over all observations collected up to round $t$:
\begin{equation}\label{eq:batch-update-baseline}
\begin{aligned} 
\hat\theta_{t+1,i} = \argmin_{\|\theta\|_2\le S} \sum_{\tau=1}^{t} \ell_{\tau,i}(\theta) + \frac{1}{2}\|\theta\|_{\Lambda}^{2},
\end{aligned} 
\end{equation} 
where $\ell_{\tau,i}(\theta)= b_i\!\left(f_i(\X_\tau)^\top\theta\right)-y_{\tau,i} \cdot f_i(\X_\tau)^\top\theta$ and $b_i(\cdot)$ is the cumulant function associated with the link function $\mu_i(\cdot)$, i.e., $b_i'(x) = \mu_i(x)$. Since the estimator is recomputed from scratch at every round, the cumulative update cost grows quadratically with the time horizon, making this approach computationally inefficient for large-scale online learning.

The corresponding confidence radius is given by
\begin{equation}\label{radius-batch}
\begin{aligned}
\beta_{t+1}=L_\mu R\sqrt{L_{t,k}+2\log\left(m/\delta\right)}+\beta_1,
\end{aligned}
\end{equation}
where $L_{t,k}=k\log\left(1+\frac{t}{k}\right)+\frac{c_\mu t}{2\lambda_\perp}$.

We now present the regret guarantee of \textsc{Scalar-LowGLM}. The following theorem shows that, with an appropriate exploration length and anisotropic regularization, the algorithm achieves sublinear regret for every objective.

\begin{thm}
\label{thm:scalar_lowglm}
Suppose that Assumptions~\ref{ass:low_rank}--\ref{ass:exploration} hold. For each objective $i\in[m]$, let $D_{rr,i}$ denote the $r$-th largest singular value of the objective-specific parameter matrix $\Theta_i^*$, and define $D_{rr}=\min_{i\in[m]}D_{rr,i}$. Run Algorithm~\ref{alg:scalarized_batch_lowglm} with
\[
T_1 \asymp\frac{\sqrt{M(d_1+d_2)rT\log((d_1+d_2)m/\delta)}}{D_{rr}},
\]
and set
\[
\lambda_0=\max\{1,c_\mu/2\},\quad \lambda_\perp=\frac{c_\mu T}{k\log\!\left(1+c_\mu T/(k\lambda_0)\right)},
\]
where $k=(d_1+d_2)r-r^2$. Furthermore, set
\[
S_{\perp}=\sqrt{\frac{M(d_1+d_2)r\log(m(d_1+d_2)/\delta)}{D_{rr}^2 T}}.
\]
Then, with probability at least $1-2\delta$, for every objective $i\in[m]$, the cumulative regret satisfies
\[
R_i(T)=
\widetilde O\!\left(W^{\rm sca}\cdot(d_1+d_2)r\sqrt{T}\right),
\]
where $W^{\rm sca}=\sum_{i=1}^m (1+w)^{i-1}$.
\end{thm}
\begin{rem}
\textnormal{
When $m=1$, Theorem~\ref{thm:scalar_lowglm} indicates \textsc{Scalar-LowGLM} achieves a regret bound of $\widetilde O((d_1+d_2)r\sqrt{T})$, matching the rate in single-objective settings \citep{Kang:2022}. Compared with vectorizing matrix arms and applying lexicographic linear bandit methods \citep{AAAI:2025:Xue}, our bound replaces the ambient dimension $d_1d_2$ with the intrinsic low-rank dimension $(d_1+d_2)r$. The factor $W^{\rm sca}$ captures the cost of fixed scalarization: when $w=0$, we have $W^{\rm sca}=m$ and hence only linear growth in the number of objectives; when $w>0$, $W^{\rm sca}=\frac{(1+w)^m-1}{w}$, which grows geometrically with the number of objectives.
}
\end{rem}



\subsection{Lexicographic Online Method: \textsc{Lexi-LowGLM}}

\begin{algorithm}[t]
  \caption{\textsc{Lexi-LowGLM}}
  \label{alg:lexi_ucb}
  \begin{algorithmic}[1]
  \REQUIRE $T, T_1, \delta, r, w, \lambda_0, \lambda_\perp, S_\perp$
  \STATE Run Algorithm~\ref{alg:subspace_estimation} for $T_1$ rounds and obtain $\{f_i(\cdot)\}_{i=1}^m$
  \STATE Set $p=d_1d_2$ and $k=(d_1+d_2)r-r^2$
  \STATE Set $\Lambda=\mathrm{diag}(\lambda_0 \I_k,\lambda_\perp \I_{p-k})$ 
  \STATE Initialize $\V_{1,i}=\Lambda$ and $\hat{\theta}_{1,i}=0$ for all $i\in[m]$, and set the confidence radius $\beta_1=L_\mu\cdot(\sqrt{\lambda_0}S+\sqrt{\lambda_\perp}S_\perp)$
  \STATE Initialize the candidate arm set $\SX_{1}=\SX$
  \FOR{$t=1,\ldots,T-T_1$}
    \STATE Compute $\hat{y}_{t,i}(\X)$ and $c_{t,i}(\X)$ for all $\X\in\SX_t$ by Eq.~\eqref{eq:predicted_reward_and_uncertainty}
    \STATE Select $(\X_t,i_t)=\arg\max_{\X\in\SX_t,i\in[m]}c_{t,i}(\X)$
    \STATE Set $\SX_t^{0}=\SX_{t}$
    \FOR{$i=1,2,\ldots,m$}
        \STATE $\X_{t,i}=\arg\max_{\X\in\SX_t^{i-1}} \hat{y}_{t,i}(\X)$
        \STATE
        $
          \SX_t^{i}=\{
          \X\in\SX_t^{i-1}:\hat{y}_{t,i}(\X_{t,i})-\hat{y}_{t,i}(\X)\leq W_i \cdot c_{t,i_t}(\X_t)\}
        $ with $W_i=2+4w+\cdots+4w^{i-1}$
    \ENDFOR
    \STATE Play $\X_t$ and observe $y_t=(y_{t,1},y_{t,2},\ldots,y_{t,m})$
    \FOR{$i=1,2,\ldots,m$}
      \STATE Compute the gradient $\nabla \ell_{t,i}(\hat{\theta}_{t,i})$ by Eq.~\eqref{gradient}
      \STATE Update the estimator $\hat{\theta}_{t+1,i}$ by Eq.~\eqref{update-online}
    \ENDFOR
    \STATE Compute the confidence radius $\beta_{t+1}$ by Eq.~\eqref{radius-online}
    \STATE Set $\SX_{t+1}=\SX_t^{m}$
  \ENDFOR
\end{algorithmic}
\end{algorithm}

Algorithm~\ref{alg:lexi_ucb} is our proposed online method. It shares the same subspace-estimation and low-rank feature initialization steps as \textsc{Scalar-LowGLM}: first run Algorithm~\ref{alg:subspace_estimation} to construct the transformed feature maps $\{f_i(\cdot)\}_{i=1}^m$, and then initialize the anisotropic regularization matrix $\Lambda$, the covariance matrices $\{\V_{1,i}\}_{i=1}^m$, the estimators $\{\hat\theta_{1,i}\}_{i=1}^m$, and the confidence radius. Beyond this shared initialization, \textsc{Lexi-LowGLM} maintains an active candidate set $\SX_t$, initialized as $\SX_1=\SX$, and progressively refines it throughout the learning process.

Unlike \textsc{Scalar-LowGLM}, which scalarizes the vector-valued reward into a single score, \textsc{Lexi-LowGLM} explicitly preserves the lexicographic preference by sequentially eliminating statistically suboptimal arms objective by objective. At each round, \textsc{Lexi-LowGLM} computes the predicted reward and confidence width for every candidate arm under each objective using Eq.~\eqref{eq:predicted_reward_and_uncertainty}, and selects the most uncertain arm-objective pair in the current candidate set:
\begin{equation*}
(\X_t,i_t)\in\argmax_{\X\in\SX_t,\,i\in[m]}c_{t,i}(\X).
\end{equation*}
The confidence width $c_{t,i_t}(\X_t)$ serves as a common tolerance threshold for the subsequent lexicographic filtering procedure.

To preserve the lexicographic preferences, the algorithm progressively filters the candidate arm set according to the objective priority. Starting from $\SX_t^{0}=\SX_t$, for each objective $i=1,2,\ldots,m$, \textsc{Lexi-LowGLM} first selects the empirically best arm in the current candidate set:
\begin{equation*}
\X_{t,i}=\argmax_{\X\in\SX_t^{i-1}}\hat{y}_{t,i}(\X).
\end{equation*}
It then removes arms whose estimated rewards are significantly inferior to that of $\X_{t,i}$ under the $i$-th objective. Specifically, the candidate set is updated as
\begin{equation*}
\SX_t^{i}=\left\{
\X\in\SX_t^{i-1}:\hat{y}_{t,i}(\X_{t,i}) - \hat{y}_{t,i}(\X) \leq W_i \cdot c_{t,i_t}(\X_t) 
\right\},
\end{equation*}
where $W_i=2+4w+\cdots+4w^{i-1}$. 

The tolerance factor $W_i$ accounts for the cumulative trade-off induced by the lexicographic structure. Since the candidate set is refined sequentially from objective $1$ to objective $m$, higher-priority objectives are enforced before lower-priority ones. Consequently, the final candidate set $\SX_t^m$ contains the arms that remain promising across all objectives. After all objectives have been processed, \textsc{Lexi-LowGLM} plays arm $\X_t$ and observes the reward vector $y_t=(y_{t,1},y_{t,2},\ldots,y_{t,m})$.

Unlike \textsc{Scalar-LowGLM}, which recomputes the estimator from all historical observations, \textsc{Lexi-LowGLM} updates each objective-specific estimator via an online Newton-type proximal step. For each objective $i\in[m]$, it first computes the gradient of the instantaneous loss at the current estimator,
\begin{equation}\label{gradient}
\nabla \ell_{t,i}(\hat{\theta}_{t,i})=(\mu_i(f_i(\X_t)^\top\hat{\theta}_{t,i})-y_{t,i})\cdot f_i(\X_t),
\end{equation}
and updates the covariance matrix by
\begin{equation*}
\V_{t+1,i} = \V_{t,i} + \frac{c_\mu}{2}f_i(\X_t)f_i(\X_t)^\top.
\end{equation*}
The new estimator is then obtained by solving the following constrained proximal problem:
\begin{equation}\label{update-online}
  \begin{aligned}
    \hat{\theta}_{t+1,i}=\argmin_{\norm{\theta}_2\leq S}\frac{\norm{\theta-\hat{\theta}_{t,i}}^2_{\V_{t+1,i}}}{2}+\langle\theta,\nabla \ell_{t,i}(\hat{\theta}_{t,i})\rangle.
  \end{aligned}
\end{equation}
Since each update depends only on the current observation, it avoids repeated batch optimization and substantially improves computational efficiency.

After updating all objective-specific estimators, the confidence radius is set to
\begin{equation}\label{radius-online}
\begin{aligned}
\beta_{t+1}=L_\mu\left(4(U+R)\sqrt{\frac{L_{t,k}+L_{t,\delta}}{c_\mu}}+\sqrt{\frac{c_\mu}{2}}\right)+\beta_1,
\end{aligned}
\end{equation}
where $L_{t,\delta}=\log\left(\frac{m\sqrt{1+4S^2t}}{\delta}\right)$.

Finally, the candidate set for the next round is updated as $\SX_{t+1}=\SX_t^m$. In this way, Algorithm~\ref{alg:lexi_ucb} gradually eliminates statistically suboptimal arms while updating the objective-specific estimators in the estimated reduced feature spaces.

We next establish the regret guarantee of \textsc{Lexi-LowGLM}. The following theorem shows that, despite the sequential lexicographic filtering and online estimator updates, the algorithm achieves sublinear regret for every objective.

\begin{thm}\label{thm:lexi_lowglm_ucb}
Suppose that Assumptions~\ref{ass:low_rank}--\ref{ass:exploration} hold, and run Algorithm~\ref{alg:lexi_ucb} with the same parameters as specified in Theorem~\ref{thm:scalar_lowglm}. Then, with probability at least $1-2\delta$, for every objective $i\in[m]$, the cumulative regret satisfies
\begin{equation*}
\begin{aligned}
R_i(T)=\widetilde O\left(W_i^{\rm lex}\sqrt{m}\cdot(d_1+d_2)r\sqrt{T}\right),
\end{aligned}
\end{equation*}
where $W_i^{\rm lex} = 1+w+\cdots+w^{i-1}$.
\end{thm}
\begin{rem}
\textnormal{
Theorem~\ref{thm:lexi_lowglm_ucb} shows that \textsc{Lexi-LowGLM} matches the horizon and effective-dimension dependence achieved in the single-objective setting \citep{Kang:2022}, while jointly learning all objectives. Compared with the scalarized baseline in Theorem~\ref{thm:scalar_lowglm}, the improvement lies in the objective-dependent factor. In particular, the regret for the highest-priority objective is independent of the trade-off parameter $w$. When $w=0$, we have $W_i^{\rm lex}=1$ for all $i\in[m]$, whereas $W^{\rm sca}=m$. Thus, lexicographic filtering improves the dependence on the number of objectives from $m$ to $\sqrt m$. When $w>0$, $W_i^{\rm lex}$ remains more refined than $W^{\rm sca}$ since $W_i^{\rm lex}$ grows with the prefix length $i$ rather than the total number of objectives $m$, and uses powers of $w$ instead of powers of $1+w$. This highlights the benefit of explicitly exploiting the sequential structure of lexicographic preferences. In addition, \textsc{Lexi-LowGLM} uses online estimator updates instead of recomputing a batch estimator from all historical samples, which improves computational efficiency and makes it more suitable for large-scale online learning.
}
\end{rem}

\section{Experiments}

We conduct numerical experiments to evaluate the statistical and computational performance of the proposed methods. In particular, we compare their objective-wise regret and running time on synthetic lexicographic generalized low-rank matrix bandit instances with different matrix ranks.

\begin{figure*}[tb]
  \centering 
  \setlength{\abovecaptionskip}{0cm}
  \setlength{\belowcaptionskip}{0cm}
  \includegraphics[width=1\textwidth]{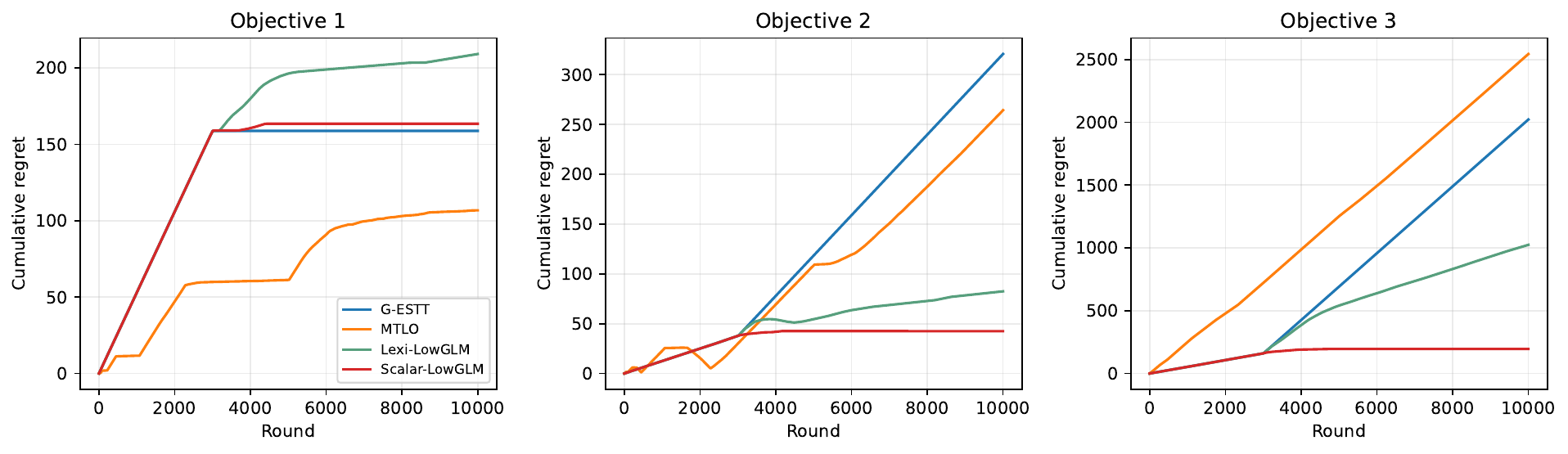}
  \caption{Regret comparison of our algorithms versus \textsc{G-ESTT} and \textsc{MTLO} for rank $1$.}
\label{fig:rank_1}
\end{figure*}

\paragraph{Baselines.}
We compare \textsc{Scalar-LowGLM} and \textsc{Lexi-LowGLM} with two representative baselines: \textsc{G-ESTT} \citep{Kang:2022} and \textsc{MTLO} \citep{AAAI:2025:Xue}. \textsc{G-ESTT} is designed for single-objective generalized low-rank matrix bandits and is therefore applied to the highest-priority objective. It exploits the low-rank matrix structure but does not account for multiple objectives. \textsc{MTLO} is a lexicographic linear bandit algorithm applied to the vectorized matrix features $\operatorname{vec}(X)\in\R^{d_1d_2}$. It captures lexicographic preferences but operates in the ambient dimension $d_1d_2$ without exploiting the underlying low-rank structure.

We set the horizon to $T=10{,}000$, the exploration length to $T_1=3{,}000$, and consider ranks $r\in\{1,2\}$. Each experiment is repeated over $10$ independent trials, and the regret curves report the average performance across trials. Full details of the synthetic instances and hyperparameter settings are provided in the appendix. The results and discussion for the rank-two setting are also deferred to the appendix.

Figure~\ref{fig:rank_1} compares the objective-wise cumulative regret of all algorithms. \textsc{G-ESTT} effectively controls the regret on Objective~1, but its regret grows nearly linearly on Objectives~2 and~3. This behavior is expected because \textsc{G-ESTT} is designed for single-objective learning, without explicitly accounting for the lower-priority objectives. \textsc{MTLO} achieves the smallest regret on Objective~1 during the early rounds because it operates directly on the fixed finite arm set and avoids the $T_1$ rounds of continuous subspace exploration required by the low-rank methods. Nevertheless, its regret continues to increase over time because of the model mismatch: \textsc{MTLO} is developed for linear bandits, whereas our instances follow a generalized linear reward model. 

Both proposed methods achieve substantially smaller regret on Objectives~2 and~3 than \textsc{G-ESTT} and \textsc{MTLO}. The relatively larger Objective~1 regret of \textsc{Lexi-LowGLM} is partly attributable to the initial subspace-exploration phase. \textsc{Scalar-LowGLM} stabilizes shortly after this exploration phase and attains the smallest empirical regret on the two lower-priority objectives. Although the theoretical bound of \textsc{Lexi-LowGLM} has a more favorable objective-dependent factor $W_i^{\mathrm{lex}}$, this does not imply that its finite-horizon regret must be uniformly smaller than that of \textsc{Scalar-LowGLM}. The bounds suppress constants and additional terms arising from confidence radii, regularization, and subspace-estimation error. Hence, the theoretical comparison concerns worst-case asymptotic guarantees rather than a strict empirical ordering at $T=10{,}000$. The observed sublinear growth of the regret curves is nevertheless consistent with the theoretical guarantees.

\begin{table}[h]
  \centering
  \caption{Mean total wall-clock time in seconds.}
  \label{tab:runtime}
  \begin{tabular}{llll}
    \toprule
    Algorithm & Runtime & Algorithm & Runtime\\
    \midrule
    \textsc{G-ESTT}        & 87.899 &   \textsc{Scalar-LowGLM} & 227.555\\
    \textsc{MTLO}          & 4.349 & \textsc{Lexi-LowGLM}   & \textbf{4.070} \\
    \bottomrule
  \end{tabular}
\end{table}

Table~\ref{tab:runtime} reports the mean wall-clock time required to complete a rank-one trial over $10{,}000$ rounds. \textsc{Lexi-LowGLM} finishes in only $4.070$ seconds, making it approximately $21\times$ faster than \textsc{G-ESTT} and $55\times$ faster than \textsc{Scalar-LowGLM}. This gap reflects their different estimator-update mechanisms: \textsc{G-ESTT} repeatedly refits a batch estimator for the highest-priority objective, while \textsc{Scalar-LowGLM} performs batch refitting for all three objectives. By contrast, \textsc{Lexi-LowGLM} updates each estimator online using only the current observation, thereby avoiding repeated processing of the full history. \textsc{MTLO} and \textsc{Lexi-LowGLM} have similar running times because both rely on online updates. However, \textsc{MTLO} is designed for linear rewards and does not accommodate the generalized linear setting, resulting in linear regret on Objectives 2 and 3. Overall, \textsc{Lexi-LowGLM} achieves the lowest mean runtime while substantially outperforming \textsc{MTLO} on the lower-priority objectives.

\section{Conclusion and Future Work}\label{sec:conclusion}
We studied lexicographic generalized low-rank matrix bandits, extending generalized low-rank matrix bandits from scalar rewards to multiple prioritized objectives. We proposed two algorithms: \textsc{Scalar-LowGLM} and \textsc{Lexi-LowGLM}. \textsc{Scalar-LowGLM} serves as a natural scalarized batch baseline, while \textsc{Lexi-LowGLM} combines objective-specific subspace estimation, lexicographic filtering, and online Newton-type updates. Theoretically, we established objective-wise regret bounds of order $\widetilde O(W_i^{\rm lex}\sqrt{m}(d_1+d_2)r\sqrt{T})$, showing that the regret depends on the intrinsic low-rank dimension $r$. Moreover, compared with repeated batch re-estimation \citep{Kang:2022}, \textsc{Lexi-LowGLM} reduces the cumulative estimator-update complexity from $O(T^2)$ to $O(T)$, substantially improving computational efficiency.

Future work will focus on extending the framework to rank-adaptive learning, and more challenging feedback models such as non-stationary, delayed, or heavy-tailed rewards. Another interesting direction is to derive matching lower bounds for lexicographic low-rank matrix bandits.

\bibliographystyle{plainnat}
\bibliography{ref}



\newpage
\appendix


\section{Experimental Setup and Rank-Two Results}

\subsection{Detailed Experimental Setup}

We set $d_1=d_2=10$, consider $m=3$ objectives, and examine matrix ranks $r\in\{1,2\}$. Each problem instance contains $K=10$ fixed matrix arms. For each $k\in[K]$, define the objective-specific latent scores
\begin{equation*}
\begin{aligned}
g_1(k)
&=
1-\min\{|k-1|,|k-4|,|k-8|\},\\
g_2(k)
&=
1-\min\{|k-3|,|k-7|\},\\
g_3(k)
&=
1-|k-8|.
\end{aligned}
\end{equation*}
Thus, arms $1$, $4$, and $8$ maximize the first objective. Among these arms, $4$ and $8$ remain optimal under the second objective, while the third objective uniquely identifies arm $8$. Hence, arm $8$ is the unique lexicographically optimal arm.

Let $\{e_j\}_{j=1}^{10}$ denote the canonical basis of $\R^{10}$. For each objective $i\in[m]$, define the unknown parameter matrix as
\begin{equation*}
\Theta_i^\star
=
\frac{1}{\sqrt r}
\sum_{\ell=1}^{r}
e_{i+3(\ell-1)}e_\ell^\top.
\end{equation*}
By construction,
$
\operatorname{rank}(\Theta_i^\star)=r
\text{ and }
\|\Theta_i^\star\|_{\mathrm F}=1.
$
The $k$-th fixed arm is defined as
\begin{equation*}
\X_k
=
\frac{1}{\kappa_r r}
\sum_{\ell=1}^{r}
\sum_{i=1}^{m}
g_i(k)e_{i+3(\ell-1)}e_\ell^\top,
\end{equation*}
where
\begin{equation*}
\kappa_r
=
\max_{k\in[K]}
\sqrt{
\frac{1}{r}
\sum_{i=1}^{m}g_i(k)^2
}.
\end{equation*}
This normalization ensures that $\|\X_k\|_{\mathrm F}\leq1$ for every $k\in[K]$. Moreover,
\begin{equation*}
\langle \X_k,\Theta_i^\star\rangle
=
\frac{1}{\kappa_r\sqrt r}g_i(k),
\end{equation*}
so the resulting matrix construction preserves the rankings and ties specified by the latent scores.

After selecting arm $\X_t$, the learner observes
\begin{equation*}
y_{t,i}
=
\mu\!\left(\langle \X_t,\Theta_i^\star\rangle\right)
+
\eta_{t,i},
\qquad
\mu(z)=\frac{1}{1+\exp(-z)},
\end{equation*}
where the noise variables
$
\eta_{t,i}\sim\operatorname{Uniform}[-0.1,0.1]
$
are independent across rounds and objectives.

During the first $T_1$ rounds, the low-rank methods sample arms from a continuous exploration distribution. Let
$
a=1/\sqrt{d_1d_2}.
$
For each entry, we independently draw
$
Z_{uv}\sim\operatorname{Beta}(3,3)
$
and set
$
(\X_t)_{uv}=-aZ_{uv}.
$
This construction guarantees $\|\X_t\|_{\mathrm F}\leq1$. The corresponding entrywise score function is
\begin{equation*}
S(\X_t)_{uv}
=
-\frac{2\bigl(a+2(\X_t)_{uv}\bigr)}
{(\X_t)_{uv}\bigl(a+(\X_t)_{uv}\bigr)}.
\end{equation*}
The exploration radius is chosen such that $\X_8$ remains the comparator among the fixed arms. After the exploration phase, the low-rank methods select from the ten fixed arms, whereas \textsc{MTLO} operates on the fixed candidate set from the beginning.

We set the horizon to $T=10000$ and the exploration length to $T_1=3000$. The exploration radius, bounded-noise half-width, confidence-radius scaling factor, Beta shape parameter, and Stein nuclear-penalty scaling factor are set to $1$, $0.1$, $0.004$, $3$, and $10^{-5}$, respectively. Each experiment is repeated over $10$ independent trials, with matched random seeds across methods. The regret curves report the mean over all trials. Runtime is measured as the total wall-clock time required to complete one $10000$-round trial, including exploration, arm selection, reward generation, and estimator updates.

\subsection{Additional Rank-Two Results}
\begin{figure}[h]
  \centering
  \includegraphics[width=0.95\textwidth]{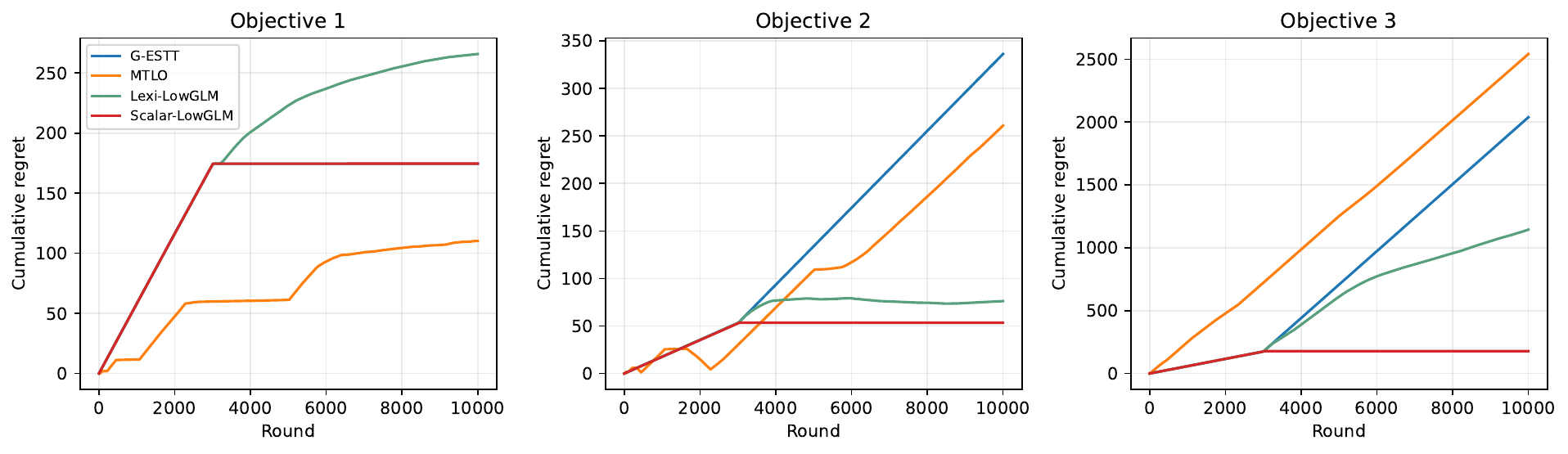}
  \caption{Regret comparison of our algorithms versus \textsc{G-ESTT} and \textsc{MTLO} for rank $2$.}
  \label{fig:rank_2}
\end{figure}

Figure~\ref{fig:rank_2} shows that the qualitative ordering observed for rank one largely persists when the rank increases to two. \textsc{Scalar-LowGLM} reaches nearly flat curves on all three objectives after the exploration phase and achieves the smallest lower-priority regret. Its first-objective curve almost overlaps with that of \textsc{G-ESTT}. Both methods share the same subspace-estimation stage and batch GLM update for Objective 1, and, on this instance, the additional scalarized objectives do not change the arm preferred by \textsc{Scalar-LowGLM} on the highest-priority objective. The batch refitting of all three objective models then enables \textsc{Scalar-LowGLM} to identify arm $8$ and correct temporary estimation errors using the complete observation history.

In contrast, \textsc{G-ESTT} accumulates almost linear regret on Objectives 2 and 3. Under Objective 1, arms $1$, $4$, and $8$ are tied, so a method that observes only the primary objective has no information with which to resolve this set according to the lower-priority objectives. \textsc{MTLO} also maintains relatively small Objective-1 regret because it starts directly on the fixed arm set and prioritizes the first objective, thereby avoiding the initial continuous exploration cost. However, it must learn in the ambient $d_1d_2$-dimensional space under a linear reward model, while the data are generated by a generalized low-rank model. This model and dimension mismatch slows its sequential candidate refinement; consequently, its Objective-2 regret continues to grow and its Objective-3 regret is the largest among all methods.

\textsc{Lexi-LowGLM} substantially reduces the lower-priority regret relative to \textsc{G-ESTT} and \textsc{MTLO}. Its Objective-2 curve approaches a plateau, indicating that objective-specific low-rank learning and lexicographic filtering successfully eliminate most arms that are inconsistent with the first two priority levels. Its Objective-3 curve also has a visibly smaller and decreasing slope, but it does not fully flatten within $10000$ rounds. Rank two is statistically harder than rank one because the effective transformed dimension increases from $(d_1+d_2)r-r^2=19$ to $36$. The resulting larger subspace and parameter uncertainty keeps multiple candidate arms active for longer. Hence \textsc{Lexi-LowGLM} continues to explore among arms that are nearly indistinguishable under the higher-priority objectives instead of selecting arm $8$ exclusively.

\section{Proof of Theorem~\ref{thm:scalar_lowglm}}

\paragraph{Proof roadmap.}
The proof has three components. First, we use the Stein-type estimator to
control the objective-specific subspace estimation error and then translate
this error into a bound on the coordinates outside the estimated low-rank
subspace. Second, we establish a uniform confidence bound for the batch
generalized linear estimators under anisotropic regularization. Third, we show
that the scalarized regret dominates every objective-wise regret, apply
optimism and the elliptical-potential argument, and finally substitute the
prescribed choices of $T_1$, $\lambda_\perp$, and $S_\perp$.

\paragraph{Step 1: Subspace estimation and tail-coordinate control.}

\begin{lem}\label{lem:subspace_estimation}
Suppose that Assumptions~\ref{ass:bounded_norm}--\ref{ass:bounded_noise} and \ref{ass:exploration} hold. Let $\X_1,\ldots,\X_{T_1}$ be sampled independently from $\D$ over $\SX$. For each objective $i\in[m]$, let $\widehat{\Theta}_i$ be the solution to the nuclear-norm regularized problem~\eqref{eq:subspace_estimation}. Set $\nu=\sqrt{ \frac{2\log(2m(d_1+d_2)/\delta)}{(4R^2+U^2)MT_1(d_1+d_2)}}$ and $\lambda_{T_1}=4\sqrt{\frac{2(4R^2+U^2)M(d_1+d_2)\log(2m(d_1+d_2)/\delta)}{T_1}}$. Then, with probability at least $1-\delta$, the following bound holds for all $i\in[m]$:
\begin{equation*}
\|\widehat{\Theta}_i-\mu^*_i\Theta^*_i\|^2_F \leq \frac{C_1M(d_1+d_2)r\log(2m(d_1+d_2)/\delta)}{T_1}
\end{equation*}
for $C_1 = 36(4R^2+U^2)$ and $\mu^*_i = \E[\mu'_i(\langle X,\Theta_i^*\rangle)]\geq c_\mu>0$.
\end{lem}
\noindent\textbf{Proof.} Under Assumptions~\ref{ass:bounded_norm}--\ref{ass:bounded_noise} and~\ref{ass:exploration}, the observations associated with objective $i$ constitute a low-rank generalized linear model satisfying the conditions of \citet[Theorem 4.1]{Kang:2022}. Applying this theorem to objective $i$ with confidence level $\delta'=\delta/m$ gives
\[
    \Pr(\mathcal E_i)\ge 1-\delta',
\]
where
\[
    \mathcal E_i
    :=
    \left\{
    \left\|
        \widehat{\Theta}_i-\mu_i^*\Theta_i^*
    \right\|_F^2
    \le
    \frac{
        C_1M(d_1+d_2)r
        \log\!\left(2(d_1+d_2)/\delta'\right)
    }{
        T_1
    }
    \right\},
\]
with $C_1=36(4R^2+U^2)$ and $\mu_i^*=\E_{\X\sim\D}\left[\mu'_i\bigl(\langle \X,\Theta_i^*\rangle\bigr)\right]$.

Since $\delta'=\delta/m$, we have
\[
    \log\left(\frac{2(d_1+d_2)}{\delta'}\right)
    =
    \log\left(\frac{2m(d_1+d_2)}{\delta}\right).
\]
Consequently, the regularization parameters become
\[
    v
    =
    \sqrt{
        \frac{
            2\log\!\left(2m(d_1+d_2)/\delta\right)
        }{
            (4R^2+U^2)MT_1(d_1+d_2)
        }
    }
\]
and
\[
    \lambda_{T_1}
    =
    4
    \sqrt{
        \frac{
            2(4R^2+U^2)M(d_1+d_2)
            \log\!\left(2m(d_1+d_2)/\delta\right)
        }{
            T_1
        }
    }.
\]
Finally, by the union bound,
\[
\begin{aligned}
    \Pr\left(\bigcap_{i=1}^m\mathcal E_i\right)
    =
    1-
    \Pr\left(\bigcup_{i=1}^m\mathcal E_i^c\right)
    \ge
    1-\sum_{i=1}^m\Pr(\mathcal E_i^c)
    \ge
    1-\sum_{i=1}^m\frac{\delta}{m}
    =
    1-\delta.
\end{aligned}
\]
Thus, with probability at least $1-\delta$, the claimed estimation bound holds simultaneously for all $i\in[m]$.

Furthermore, under the assumed uniform lower bound $\mu'_i(z)\ge c_\mu$ over the relevant parameter domain,
\[
    \mu_i^*
    =
    \E\left[
        \mu'_i
        \bigl(\langle X,\Theta_i^*\rangle\bigr)
    \right]
    \ge c_\mu>0
\]
for every $i\in[m]$. This completes the proof.
$\hfill\square$

After acquiring the estimated $\widehat{\Theta}_i$ in Algorithm~\ref{alg:subspace_estimation}, we can obtain the corresponding
SVD as
\begin{equation*}
\widehat{\Theta}_i=[\widehat{\U}_i,\widehat{\U}_{i,\perp}]\widehat{\DM}_i[\widehat{\V}_i,\widehat{\V}_{i,\perp}]^\top,
\end{equation*}
where
\begin{equation*}
\widehat{\U}_i\in\mathbb{R}^{d_1\times r},\quad
\widehat{\U}_{i,\perp}\in\mathbb{R}^{d_1\times (d_1-r)},\quad
\widehat{\V}_i\in\mathbb{R}^{d_2\times r},\quad
\widehat{\V}_{i,\perp}\in\mathbb{R}^{d_2\times (d_2-r)}.
\end{equation*}
And we assume the SVD of the matrix $\Theta^*_i$ can be represented as
\begin{equation*}
\Theta^*_i = \U_i\DM_i\V_i^\top,
\end{equation*}
where $\U_i\in\mathbb{R}^{d_1\times r}$ and $\V_i\in\mathbb{R}^{d_2\times r}$. To transform the original generalized matrix bandits into generalized linear bandit problems, we penalize those covariates that are complementary to $\widehat{\U}_i$ and $\widehat{\V}_i$. Specifically, we could orthogonally rotate the inherent parameter $\Theta_i^*$ as
\begin{equation*}
\Theta'_i = [\widehat{\U}_i,\widehat{\U}_{i,\perp}]^\top \Theta_i^* [\widehat{\V}_i,\widehat{\V}_{i,\perp}].
\end{equation*}

Define the total dimension and the effective dimension as
\begin{equation*}
p = d_1d_2, \quad k = d_1d_2-(d_1-r)(d_2-r).
\end{equation*}
For each objective $i\in[m]$, let $D_{rr,i}$ denote the $r$-th largest singular value of the objective-specific parameter matrix $\Theta_i^*$, and define $D_{rr}=\min_{i\in[m]}D_{rr,i}$. Then, for the true parameter $\theta^*_i$ after
transformation, we denote the last $p-k$ entries as $\theta^*_{i,k+1:p}$, such that
\begin{equation*}
\theta^*_{i,k+1:p} = \operatorname{vec}\left((\Theta'_i)_{r+1:d_1,\;r+1:d_2}\right).
\end{equation*}
\begin{lem}[Subspace perturbation to tail control]
\label{lem:subspace_to_tail}
Let $A=UDV^\top$ have rank $r$ and smallest nonzero singular value
$\sigma_r(A)>0$. Let $\widehat A$ be any matrix, and let
$\widehat U$ and $\widehat V$ contain its leading $r$ left and right
singular vectors. Then
\[
\|\widehat U_\perp^\top U\|_F,\,
\|\widehat V_\perp^\top V\|_F
\le
\frac{2\|\widehat A-A\|_F}{\sigma_r(A)}
\]
and
\[
\|\widehat U_\perp^\top UDV^\top\widehat V_\perp\|_F
\le
\frac{4\|D\|_{\mathrm{op}}\|\widehat A-A\|_F^2}
{\sigma_r(A)^2}.
\]
\end{lem}

\paragraph{Proof.}
Let $\widehat A_r$ be the best rank-$r$ approximation of $\widehat A$.
Because $A$ has rank $r$, the Eckart--Young theorem gives
$\|\widehat A-\widehat A_r\|_F\le\|\widehat A-A\|_F$. Therefore,
\[
\|\widehat U_\perp^\top A\|_F
\le
\|\widehat A-A\|_F+\|\widehat A-\widehat A_r\|_F
\le 2\|\widehat A-A\|_F.
\]
On the other hand,
$\|\widehat U_\perp^\top A\|_F
\ge\sigma_r(A)\|\widehat U_\perp^\top U\|_F$.
This proves the left-subspace bound; applying the same argument to
$A^\top$ proves the right-subspace bound. Finally,
\[
\|\widehat U_\perp^\top UDV^\top\widehat V_\perp\|_F
\le
\|\widehat U_\perp^\top U\|_F
\|D\|_{\mathrm{op}}
\|V^\top\widehat V_\perp\|_F,
\]
and substituting the two subspace bounds proves the result.
$\hfill\square$

\begin{lem}\label{lem:transformed_tail_bound}
Suppose that the conditions of Lemma~\ref{lem:subspace_estimation} hold. Then, with probability at least $1-\delta$, simultaneously for all $i\in[m]$,
\[
    \left\|\theta_{i,k+1:p}^*\right\|_2
    \lesssim
    \frac{M(d_1+d_2)r}{T_1D_{rr}^2}
    \log\left(
        \frac{m(d_1+d_2)}{\delta}
    \right)
    =:S_{\perp}.
\]
\end{lem}
\paragraph{Proof.}
Set $A_i=\mu_i^*\Theta_i^*$ and
$E_i=\widehat{\Theta}_i-A_i$. Since $\mu_i^*>0$, $A_i$ and
$\Theta_i^*$ have the same singular subspaces, and
$\sigma_r(A_i)=\mu_i^*D_{i,rr}$. Moreover, the last $p-k$ coordinates
of $\theta_i^*$ are the vectorization of
$\widehat{\U}_{i,\perp}^\top\Theta_i^*\widehat{\V}_{i,\perp}$.
Lemma~\ref{lem:subspace_to_tail} therefore gives
\[
\|\theta_{i,k+1:p}^*\|_2
\le
\frac{4\|\DM_i\|_{\mathrm{op}}\|E_i\|_F^2}
{(\mu_i^*)^2D_{i,rr}^2}.
\]
By Lemma~\ref{lem:subspace_estimation}, the event
\[
\|E_i\|_F^2
\le
\frac{C_1M(d_1+d_2)r
\log(2m(d_1+d_2)/\delta)}{T_1}
\]
holds simultaneously for all objectives. Using
$\mu_i^*\ge c_\mu$, $D_{i,rr}\ge D_{rr}$, and
$\|\DM_i\|_{\mathrm{op}}\le D_{\max}$, and suppressing constants
depending on $C_1$, $c_\mu$, and $D_{\max}$, proves the claim.
$\hfill\square$


\paragraph{Step 2: Confidence bound for the batch estimators.}

The next lemma converts the batch estimation error into a reward confidence
interval that holds uniformly over rounds, objectives, and arms.

\begin{lem}\label{lem:scalar_reward_confidence}
Suppose Assumptions~\ref{ass:bounded_norm}--\ref{ass:bounded_noise} hold. Then, with probability at least $1-\delta$, for all $t\ge 1$, all objectives $i\in[m]$, and all arms $\X\in\SX$, we have
\[
\left|\mu_i\!\left(f_i(\X)^\top\theta_i^*\right)-\mu_i\!\left(f_i(\X)^\top\hat{\theta}_{t,i}\right)\right|
\le\beta_t\|f_i(\X)\|_{\V_{t,i}^{-1}},
\]
where
\[
\beta_{t+1}=L_\mu
\left(
R\sqrt{k\log\left(1+\frac{t}{k}\right)+\frac{c_\mu t}{2\lambda_\perp}
+2\log\left(\frac{m}{\delta}\right)}
+
\left(\sqrt{\lambda_0}S+\sqrt{\lambda_\perp}S_\perp\right)
\right).
\]
\end{lem}

\paragraph{Proof.} Fix an objective $i\in[m]$ and write
\[
x_{\tau,i}=f_i(\X_\tau), \qquad x=f_i(\X).
\]
For notational simplicity, we omit the subscript $i$ when there is no ambiguity. At the beginning of round $t$, the estimator $\hat\theta_{t,i}$ is computed from the observations collected in the previous $t-1$ rounds. Define
\[
g_{t,i}(\theta)
=
\sum_{\tau=1}^{t-1}
\mu_i(x_{\tau,i}^\top\theta)x_{\tau,i}
+
\Lambda\theta .
\]
Since $\hat\theta_{t,i}$ minimizes the regularized negative log-likelihood, the first-order optimality condition gives
\[
g_{t,i}(\hat\theta_{t,i})
=
\sum_{\tau=1}^{t-1}
y_{\tau,i}x_{\tau,i}.
\]
Moreover,
\[
g_{t,i}(\theta_i^*)
-
g_{t,i}(\hat\theta_{t,i})
=
-\sum_{\tau=1}^{t-1}
\eta_{\tau,i}x_{\tau,i}
+
\Lambda\theta_i^*,
\]
where
\[
\eta_{\tau,i}
=
y_{\tau,i}
-
\mu_i(x_{\tau,i}^\top\theta_i^*).
\]


Next, by the fundamental theorem of calculus,
\[
g_{t,i}(\theta_i^*)-g_{t,i}(\hat\theta_{t,i})
=
G_{t,i}(\theta_i^*-\hat\theta_{t,i}),
\]
where
\[
G_{t,i}
=
\int_0^1
\nabla g_{t,i}\!\left(
s\theta_i^*+(1-s)\hat\theta_{t,i}
\right)
ds .
\]
Since $\dot\mu_i(\cdot)\ge c_\mu$, we have
\[
G_{t,i}
\succeq
c_\mu
\sum_{\tau=1}^{t-1}
x_{\tau,i}x_{\tau,i}^\top
+
\Lambda
\succeq
\V_{t,i}=\frac{c_\mu}{2}
\sum_{\tau=1}^{t-1}
x_{\tau,i}x_{\tau,i}^\top
+
\Lambda.
\]
Using the Lipschitzness of $\mu_i(\cdot)$, we obtain
\[
\begin{aligned}
\left|
\mu_i(x^\top\theta_i^*)
-
\mu_i(x^\top\hat\theta_{t,i})
\right|
&
\le
L_\mu
\left|
x^\top(\theta_i^*-\hat\theta_{t,i})
\right|
\\
&
=
L_\mu
\left|
x^\top
G_{t,i}^{-1}
\big(
g_{t,i}(\theta_i^*)-g_{t,i}(\hat\theta_{t,i})
\big)
\right|
\\
&
\le
L_\mu\|x\|_{\V_{t,i}^{-1}}
\left\|
g_{t,i}(\theta_i^*)-g_{t,i}(\hat\theta_{t,i})
\right\|_{\V_{t,i}^{-1}} .
\end{aligned}
\]
Substituting the expression of $g_{t,i}(\theta_i^*)-g_{t,i}(\hat\theta_{t,i})$ gives
\begin{equation}\label{eq:Lip_noise}
\begin{aligned}
\left|
\mu_i(x^\top\theta_i^*)
-
\mu_i(x^\top\hat\theta_{t,i})
\right|
\le
L_\mu\|x\|_{\V_{t,i}^{-1}}
\left(
\left\|
\sum_{\tau=1}^{t-1}
\eta_{\tau,i}x_{\tau,i}
\right\|_{\V_{t,i}^{-1}}
+
\|\Lambda\theta_i^*\|_{\V_{t,i}^{-1}}
\right).
\end{aligned}
\end{equation}

We now bound the two terms inside the parentheses. Since the noise is bounded by $R$, it is conditionally $R$-sub-Gaussian. By the self-normalized martingale inequality \citep{Abbasi:2011}, with probability at least $1-\delta/m$, simultaneously for all $t\ge1$,
\begin{equation}\label{eq:subgaussian_bound}
\left\|
\sum_{\tau=1}^{t-1}
\eta_{\tau,i}x_{\tau,i}
\right\|_{\V_{t,i}^{-1}}
\le
R
\sqrt{
\log
\frac{\det(\V_{t,i})}{\det(\Lambda)}
+
2\log(m/\delta)
}.
\end{equation}
To bound the log-determinant term, Lemma~C.5 of \citet{Kang:2022},
$\lambda_0\geq c_\mu/2$, and $\log(1+x)\leq x$ give
\[
\log\frac{\det(\V_{t,i})}{\det(\Lambda)}
\leq
k\log\left(1+\frac{t}{k}\right)
+
\frac{c_\mu t}{2\lambda_\perp}.
\]
Substituting this bound into \eqref{eq:subgaussian_bound} gives
\begin{equation}\label{eq:subgaussian_bound_2}
\left\|
\sum_{\tau=1}^{t-1}
\eta_{\tau,i}x_{\tau,i}
\right\|_{\V_{t,i}^{-1}}
\le
R
\sqrt{
k\log\left(
1+\frac{t}{k}
\right)
+
\frac{c_\mu t}{2\lambda_\perp}
+
2\log\left(\frac{m}{\delta}\right)
}.
\end{equation}

For the regularization bias term, since
$
\V_{t,i}
\succeq
\Lambda,
$
we have
$
\|\Lambda\theta_i^*\|_{\V_{t,i}^{-1}}
\le
\|\theta_i^*\|_{\Lambda}.
$
By the construction of the transformed feature space,
\[
\|(\theta_i^*)_{1:k}\|_2\le S,
\qquad
\|(\theta_i^*)_{k+1:p}\|_2\le S_\perp .
\]
Thus,
\begin{equation}\label{eq:reg_bias}
\|\Lambda\theta_i^*\|_{\V_{t,i}^{-1}}\le\|\theta_i^*\|_{\Lambda}
\le
\sqrt{\lambda_0}S
+
\sqrt{\lambda_\perp}S_\perp.
\end{equation}

Taking the Eq.~\eqref{eq:subgaussian_bound_2} and Eq.~\eqref{eq:reg_bias} into Eq.~\eqref{eq:Lip_noise}, with probability at least $1-\delta/m$, for the fixed objective $i$ and all $t\ge1$, all $\X\in\SX$,
\[
\begin{aligned}
\left|
\mu_i(f_i(\X)^\top\theta_i^*)
-
\mu_i(f_i(\X)^\top\hat\theta_{t,i})
\right|\le
L_\mu
\Bigg(
R\sqrt{
k\log\left(1+\frac{t}{k}\right)
+
\frac{c_\mu t}{2\lambda_\perp}
+
2\log\left(\frac{m}{\delta}\right)
}
+
\left(
\sqrt{\lambda_0}S
+
\sqrt{\lambda_\perp}S_\perp
\right)
\Bigg)
\|f_i(\X)\|_{\V_{t,i}^{-1}} .
\end{aligned}
\]
Taking a union bound over all objectives $i\in[m]$ completes the proof.
$\hfill\square$

\paragraph{Step 3: Scalarized optimism and objective-wise regret.}

We first relate the objective-wise gaps to a single scalarized gap. This step
is where the factor $W^{\rm sca}$ enters the analysis.

\begin{lem}\label{lem:scalarized_dominates_objective_regret}
Suppose Assumption~\ref{ass:lex_tradeoff} holds. 
For any arm $\X\in\SX$ and objective $i\in[m]$, define the objective-wise gap
$
\Delta_i(\X)
=
\mu_i\!\left(\langle \X_*,\Theta_i^*\rangle\right)
-
\mu_i\!\left(\langle \X,\Theta_i^*\rangle\right).
$
Let $q=1+w$. Then, for every objective $i\in[m]$ and every arm $\X\in\SX$,
\[
\Delta_i(\X)
\le
\sum_{j=1}^m q^{m-j}\Delta_j(\X)
=
\sum_{j=1}^m (1+w)^{m-j}\Delta_j(\X).
\]
Consequently, for any sequence of arms $\{\X_t\}_{t=1}^{T_2}$,
\[
R_i(T_2)
=
\sum_{t=1}^{T_2}\Delta_i(\X_t)
\le
\sum_{t=1}^{T_2}
\sum_{j=1}^m
(1+w)^{m-j}\Delta_j(\X_t),
\qquad \forall i\in[m].
\]
\end{lem}

\paragraph{Proof.}
Fix an arbitrary arm $\X\in\SX$. 
For simplicity, write $\Delta_i=\Delta_i(\X)$. 
Since $\X_*$ is lexicographically optimal, no arm can strictly improve the first objective over $\X_*$. Hence,
\[
\Delta_1
=
\mu_1\!\left(\langle \X_*,\Theta_1^*\rangle\right)
-
\mu_1\!\left(\langle \X,\Theta_1^*\rangle\right)
\ge 0.
\]
Moreover, Assumption~\ref{ass:lex_tradeoff} implies that, for every $i\ge 2$,
\[
-\Delta_i
=
\mu_i\!\left(\langle \X,\Theta_i^*\rangle\right)
-
\mu_i\!\left(\langle \X_*,\Theta_i^*\rangle\right)
\le
w\max_{j\in[i-1]}\Delta_j.
\]

For $\ell\in[m]$, define
\[
P_\ell
=
\sum_{j=1}^{\ell}q^{\ell-j}\Delta_j,
\qquad
M_\ell
=
\max_{j\in[\ell]}\Delta_j .
\]
We prove by induction that
\[
P_\ell\ge M_\ell,\qquad \forall \ell\in[m].
\]
When $\ell=1$, we have
\[
P_1=\Delta_1=M_1\ge 0.
\]
Suppose $P_{\ell-1}\ge M_{\ell-1}$ holds for some $\ell\ge 2$. Then
\[
P_\ell
=
qP_{\ell-1}+\Delta_\ell.
\]
By the induction hypothesis and the trade-off condition,
\[
P_\ell
\ge
qM_{\ell-1}-wM_{\ell-1}
=
(q-w)M_{\ell-1}
=
M_{\ell-1},
\]
where we used $q=1+w$. If $\Delta_\ell\le M_{\ell-1}$, then $M_\ell=M_{\ell-1}$, and hence $P_\ell\ge M_\ell$. 
If $\Delta_\ell>M_{\ell-1}$, then $M_\ell=\Delta_\ell$. Since $M_{\ell-1}\ge \Delta_1\ge 0$ and $P_{\ell-1}\ge M_{\ell-1}$, we have
\[
P_\ell
=
qP_{\ell-1}+\Delta_\ell
\ge
\Delta_\ell
=
M_\ell.
\]
Thus, $P_\ell\ge M_\ell$ for all $\ell\in[m]$.

Taking $\ell=m$, we obtain
\[
\sum_{j=1}^m q^{m-j}\Delta_j
=
P_m
\ge
M_m
=
\max_{j\in[m]}\Delta_j.
\]
Therefore, for every $i\in[m]$,
\[
\Delta_i
\le
\max_{j\in[m]}\Delta_j
\le
\sum_{j=1}^m q^{m-j}\Delta_j.
\]
Substituting back $q=1+w$ gives
\[
\Delta_i(\X)
\le
\sum_{j=1}^m (1+w)^{m-j}\Delta_j(\X).
\]

Finally, applying this pointwise inequality to each played arm $\X_t$ and summing over $t=1,\ldots,T_2$ yields
\[
R_i(T_2)
=
\sum_{t=1}^{T_2}\Delta_i(\X_t)
\le
\sum_{t=1}^{T_2}
\sum_{j=1}^m
(1+w)^{m-j}\Delta_j(\X_t).
\]
This completes the proof.
$\hfill\square$

By Lemma~\ref{lem:scalarized_dominates_objective_regret},
\begin{equation}
\label{eq:objective_to_scalarized_regret}
\begin{aligned}
R_i(T_2)\leq\sum_{t=1}^{T_2}\sum_{i=1}^{m}(1+w)^{m-i}
\left(
  \mu_i\!\left(\left\langle X_*,\Theta_i^*\right\rangle\right)-\mu_i\!\left(\left\langle \X_t,\Theta_i^*\right\rangle\right)
\right).
\end{aligned}
\end{equation}

Consider the high-probability event in Lemma~\ref{lem:scalar_reward_confidence}, on which, for every
$t\in[T_2]$, $j\in[m]$, and $\X\in\SX$,
\[
    \left|
        \hat y_{t,j}(\X)
        -
        \mu_j\!\left(
            \left\langle \X,\Theta_j^*\right\rangle
        \right)
    \right|
    \le
    \beta_t\|f_j(\X)\|_{\V_{t,j}^{-1}}=c_{t,j}(\X).
\]
It follows that
\begin{equation*}
\label{eq:scalarized_optimism}
    \sum_{j=1}^m
    (1+w)^{m-j}
    \mu_j\!\left(
        \left\langle \X,\Theta_j^*\right\rangle
    \right)
    \le
    \mathrm{UCB}_t(\X,w)
\end{equation*}
for every $\X\in\SX$. Moreover,
\begin{align}
    \mathrm{UCB}_t(\X,w)-\sum_{j=1}^m
    (1+w)^{m-j}
    \mu_j\!\left(
        \left\langle \X,\Theta_j^*\right\rangle
    \right)
    &=
    \sum_{j=1}^m
    (1+w)^{m-j}
    \left[
        \hat y_{t,j}(\X)
        +
        c_{t,j}(\X)
        -
        \mu_j\!\left(
            \left\langle \X,\Theta_j^*\right\rangle
        \right)
    \right]
    \nonumber\\
    &\le
    2
    \sum_{j=1}^m
    (1+w)^{m-j}
    c_{t,j}(\X).
\label{eq:scalarized_confidence_width}
\end{align}

Using \eqref{eq:scalarized_optimism} and the arm-selection rule
\[
    \X_t
    \in
    \argmax_{\X\in\SX}
    \mathrm{UCB}_t(\X,w),
\]
we obtain
\begin{equation*}
\begin{aligned}
    R_i(T_2)
    &\le
    \sum_{t=1}^{T_2}
    \left[
        \mathrm{UCB}_t(\X_*,w)-\sum_{j=1}^m
        (1+w)^{m-j}
        \mu_j\!\left(
            \left\langle \X_t,\Theta_j^*\right\rangle
        \right)
    \right]\\
    &\le
    \sum_{t=1}^{T_2}
    \left[
        \mathrm{UCB}_t(\X_t,w)-\sum_{j=1}^m
        (1+w)^{m-j}
        \mu_j\!\left(
            \left\langle \X_t,\Theta_j^*\right\rangle
        \right)
    \right]\\
    &\le
    2
    \sum_{t=1}^{T_2}
    \sum_{j=1}^m
    (1+w)^{m-j}
    \beta_t
    \|f_j(\X_t)\|_{\V_{t,j}^{-1}}.
\end{aligned}    
\end{equation*}
Since $\beta_t\le\beta_T$, interchanging the two sums and applying the Cauchy--Schwarz inequality yield
\begin{align}
    R_i(T_2)
    &\le
    2\beta_T
    \sum_{j=1}^m
    (1+w)^{m-j}
    \sum_{t=1}^{T_2}
    \|f_j(\X_t)\|_{\V_{t,j}^{-1}}
    \nonumber\\
    &\le
    2\beta_T\sqrt{T_2}
    \sum_{j=1}^m
    (1+w)^{m-j}
    \left(
        \sum_{t=1}^{T_2}
        \|f_j(\X_t)\|_{\V_{t,j}^{-1}}^2
    \right)^{1/2}.
\label{eq:regret_before_elliptical_potential}
\end{align}

For each objective, the standard elliptical-potential argument
\citep[Lemma~11]{Abbasi:2011}, followed by the anisotropic
log-determinant bound of Lemma~C.5 in \citet{Kang:2022}, gives
\[
\sum_{t=1}^{T_2}\|f_i(\X_t)\|_{\V_{t,i}^{-1}}^2
\le
\frac{4}{c_\mu}
\left(
k\log\left(1+\frac{T}{k}\right)
+
\frac{c_\mu T}{2\lambda_\perp}
\right).
\]
Thus, we have
\begin{equation*}
\begin{aligned}
R_i(T_2)\leq \sum_{i=1}^{m}(1+w)^{m-i}2\beta_T\sqrt{T_2}\sqrt{\frac{4}{c_\mu} \left(k\log\left(1+\frac{T}{k}\right)+\frac{c_\mu T}{2\lambda_\perp}\right)}\\
\end{aligned}
\end{equation*}
Taking the first $T_1$ rounds and $\mu_i(\cdot)\leq U$ into the total regret, we have
\begin{equation*}
\begin{aligned}
R_i(T)&\leq 2UT_1+ \sum_{i=1}^{m}(1+w)^{m-i}4\beta_T\sqrt{T_2}\sqrt{\frac{k}{c_\mu} \log\left(1+\frac{T}{k}\right)+\frac{T}{2\lambda_\perp}}
\end{aligned}
\end{equation*}



\paragraph{Step 4: Parameter choice and completion of the proof.}

We now substitute the choices of $T_1$, $\lambda_\perp$, and $S_\perp$ into the regret bound. 
Define
\[
W^{\rm sca}
:=
\sum_{j=1}^m (1+w)^{m-j},
\qquad
A
:=
\frac{M(d_1+d_2)r}{D_{rr}^2}
\log\!\left(\frac{m(d_1+d_2)}{\delta}\right).
\]
By the choice of the exploration length $T_1\asymp\sqrt{AT}$ and
\[
S_\perp
=
\frac{(d_1+d_2)Mr}{T_1D_{rr}^2}
\log\!\left(\frac{m(d_1+d_2)}{\delta}\right)
=
\frac{A}{T_1},
\]
we have
\[
S_\perp
\asymp
\sqrt{\frac{A}{T}}.
\]

Recall that
\[
\lambda_\perp
=
\frac{c_\mu T}{k\log\!\left(1+c_\mu T/(k\lambda_0)\right)}.
\]
Therefore,
\[
\frac{c_\mu T}{2\lambda_\perp}
=
\frac{k}{2}
\log\!\left(1+\frac{c_\mu T}{k\lambda_0}\right),
\]
and hence the confidence radius satisfies
\[
\begin{aligned}
\beta_T
&\leq
L_\mu
\left(
R\sqrt{
k\log\!\left(1+\frac{T}{k}\right)
+
\frac{c_\mu T}{2\lambda_\perp}
+
2\log\!\left(\frac{m}{\delta}\right)}
+
\sqrt{\lambda_0}S
+
\sqrt{\lambda_\perp}S_\perp
\right)
\\
&=
\widetilde O\left(
L_\mu
\left(
R\sqrt{k}
+
\sqrt{\lambda_0}S
+
\sqrt{\lambda_\perp}S_\perp
\right)
\right).
\end{aligned}
\]
Furthermore,
\[
\sqrt{\lambda_\perp}S_\perp
=
\sqrt{
\frac{c_\mu T}{
k\log\!\left(1+c_\mu T/(k\lambda_0)\right)}
}
\cdot
\sqrt{\frac{A}{T}}
=
\widetilde O\left(
\sqrt{\frac{c_\mu A}{k}}
\right).
\]
Thus,
\[
\beta_T
=
\widetilde O\left(
L_\mu
\left(
R\sqrt{k}
+
\sqrt{\lambda_0}S
+
\sqrt{\frac{c_\mu A}{k}}
\right)
\right).
\]

On the other hand, the elliptical-potential term in the regret bound can be simplified as
\[
\begin{aligned}
\sqrt{
\frac{k}{c_\mu}\log\!\left(1+\frac{T}{k}\right)
+
\frac{T}{2\lambda_\perp}
}
&=
\sqrt{
\frac{k}{c_\mu}\log\!\left(1+\frac{T}{k}\right)
+
\frac{k}{2c_\mu}
\log\!\left(1+\frac{c_\mu T}{k\lambda_0}\right)
}
\\
&=
\widetilde O\left(
\sqrt{\frac{k}{c_\mu}}
\right).
\end{aligned}
\]
Using $T_2\le T$, we obtain
\[
\begin{aligned}
R_i(T)
&\le
2UT_1
+
4W^{\rm sca}\beta_T\sqrt{T_2}
\sqrt{
\frac{k}{c_\mu}\log\!\left(1+\frac{T}{k}\right)
+
\frac{T}{2\lambda_\perp}
}
\\
&=
\widetilde O\left(
U\sqrt{AT}
+
W^{\rm sca}L_\mu
\left(
\frac{Rk}{\sqrt{c_\mu}}
+
S\sqrt{\frac{\lambda_0 k}{c_\mu}}
+
\sqrt{A}
\right)
\sqrt{T}
\right).
\end{aligned}
\]
Substituting the definition of $A$ gives
\[
R_i(T)
=
\widetilde O\left(
\left[
U
\frac{\sqrt{M(d_1+d_2)r}}{D_{rr}}
+
W^{\rm sca}L_\mu
\left(
\frac{Rk}{\sqrt{c_\mu}}
+
S\sqrt{\frac{\lambda_0 k}{c_\mu}}
+
\frac{\sqrt{M(d_1+d_2)r}}{D_{rr}}
\right)
\right]
\sqrt{T}
\right).
\]
When $U,L_\mu,R,c_\mu,\lambda_0$, and $S$ are treated as constants, and since $W^{\rm sca}\ge 1$, this simplifies to
\[
R_i(T)
=
\widetilde O\left(
W^{\rm sca}
\left(
k+
\frac{\sqrt{M(d_1+d_2)r}}{D_{rr}}
\right)
\sqrt{T}
\right).
\]
Finally, since $k=(d_1+d_2)r-r^2\asymp (d_1+d_2)r$, under the standard scaling condition
\[
\frac{\sqrt{M(d_1+d_2)r}}{D_{rr}}
\lesssim
(d_1+d_2)r,
\]
we obtain
\[
R_i(T)
=
\widetilde O\left(
W^{\rm sca}(d_1+d_2)r\sqrt{T}
\right).
\]
The proof of Theorem~\ref{thm:scalar_lowglm} is finished.
$\hfill\square$

\section{Proof of Theorem~\ref{thm:lexi_lowglm_ucb}}

\paragraph{Proof roadmap.}
The online proof separates estimation from lexicographic decision making.
First, we establish a uniform confidence bound for the online Newton-type
estimators. Second, we prove that the filtering rule never eliminates the
lexicographically optimal arm and that the arm selected in the current round
is controlled by the previous round's common uncertainty. Third, because the
most uncertain objective may change across rounds, we partition the
elliptical norms by objective, combine the resulting $m$ potential budgets,
and substitute the same parameter choices as in
Theorem~\ref{thm:scalar_lowglm}.

\paragraph{Step 1: Online estimation and reward confidence.}

\begin{lem}[One-step estimation recursion]
\label{lem:one_step_estimation}
Fix an objective $i\in[m]$ and define
$a_t=f_i(\X_t)^\top(\hat\theta_{t,i}-\theta_i^*)$. Under
Assumptions~\ref{ass:bounded_norm}--\ref{ass:bounded_noise},
\[
\begin{aligned}
\|\hat\theta_{t+1,i}-\theta_i^*\|_{\V_{t+1,i}}^2
&\le
\|\hat\theta_{t,i}-\theta_i^*\|_{\V_{t,i}}^2
-\frac{c_\mu}{2}a_t^2
+4(U+R)^2
\|f_i(\X_t)\|_{\V_{t+1,i}^{-1}}^2
+2\eta_{t,i}a_t .
\end{aligned}
\]
\end{lem}

\paragraph{Proof.} Fix an objective $i\in[m]$. 
For simplicity, we write $x_t=f_i(\X_t)$, $\hat\theta_t=\hat\theta_{t,i}$, $\theta^*=\theta_i^*$, $\V_t=\V_{t,i}$, and $\eta_t=\eta_{t,i}$. Define the instantaneous loss
\[
\ell_{t,i}(\theta)
=
-y_{t,i}x_t^\top\theta
+
\int_0^{x_t^\top\theta}\mu_i(z)\,dz .
\]
Then
\[
\nabla \ell_{t,i}(\theta)
=
\big(\mu_i(x_t^\top\theta)-y_{t,i}\big)x_t,
\]
which is consistent with Eq.~\eqref{gradient}. 
Since $\dot\mu_i(z)\ge c_\mu$, the function $\ell_{t,i}(\cdot)$ is strongly convex along the direction $x_t$. 
Thus, for any $\theta_1,\theta_2$,
\[
\ell_{t,i}(\theta_1)-\ell_{t,i}(\theta_2)
\le
\nabla\ell_{t,i}(\theta_1)^\top(\theta_1-\theta_2)
-
\frac{c_\mu}{2}
\big(x_t^\top\theta_1-x_t^\top\theta_2\big)^2 .
\]
Taking $\theta_1=\hat\theta_t$ and $\theta_2=\theta^*$ gives
\[
\ell_{t,i}(\hat\theta_t)-\ell_{t,i}(\theta^*)
\le
\nabla\ell_{t,i}(\hat\theta_t)^\top(\hat\theta_t-\theta^*)
-
\frac{c_\mu}{2}
\big(x_t^\top\hat\theta_t-x_t^\top\theta^*\big)^2 .
\]

Let $f_{t,i}(\theta)=\E[\ell_{t,i}(\theta)]$. Since $\E[y_{t,i}]=\mu_i(x_t^\top\theta^*)$, the parameter $\theta^*$ minimizes $f_{t,i}(\theta)$, and hence
\[
f_{t,i}(\hat\theta_t)-f_{t,i}(\theta^*)\ge0.
\]
Taking conditional expectation in the previous strong-convexity inequality and using the above fact yields
\[
0
\le
\nabla f_{t,i}(\hat\theta_t)^\top(\hat\theta_t-\theta^*)
-
\frac{c_\mu}{2}
\big(x_t^\top\hat\theta_t-x_t^\top\theta^*\big)^2 .
\]
By adding and subtracting $\nabla\ell_{t,i}(\hat\theta_t)$, we obtain
\[
\begin{aligned}
0
\le
\big(\nabla f_{t,i}(\hat\theta_t)-\nabla\ell_{t,i}(\hat\theta_t)\big)^\top(\hat\theta_t-\theta^*)
+
\nabla\ell_{t,i}(\hat\theta_t)^\top(\hat\theta_t-\theta^*)
-
\frac{c_\mu}{2}
\big(x_t^\top\hat\theta_t-x_t^\top\theta^*\big)^2 .
\end{aligned}
\]
Moreover,
\[
\nabla f_{t,i}(\hat\theta_t)
=
\big(\mu_i(x_t^\top\hat\theta_t)-\mu_i(x_t^\top\theta^*)\big)x_t,
\]
and therefore
\[
\nabla f_{t,i}(\hat\theta_t)-\nabla\ell_{t,i}(\hat\theta_t)
=
\eta_t x_t.
\]
Hence,
\[
\begin{aligned}
0
\le
\eta_t
\big(x_t^\top\hat\theta_t-x_t^\top\theta^*\big)
+
\nabla\ell_{t,i}(\hat\theta_t)^\top(\hat\theta_t-\theta^*)
-
\frac{c_\mu}{2}
\big(x_t^\top\hat\theta_t-x_t^\top\theta^*\big)^2 .
\end{aligned}
\]

Next, the proximal update in Eq.~\eqref{update-online} implies the standard online Newton inequality: for any feasible $\theta$,
\[
\nabla\ell_{t,i}(\hat\theta_t)^\top(\hat\theta_t-\theta)
-
\frac{1}{2}
\|\nabla\ell_{t,i}(\hat\theta_t)\|_{\V_{t+1}^{-1}}^2
\le
\frac{1}{2}
\left(
\|\hat\theta_t-\theta\|_{\V_{t+1}}^2
-
\|\hat\theta_{t+1}-\theta\|_{\V_{t+1}}^2
\right).
\]
Taking $\theta=\theta^*$ gives
\[
\begin{aligned}
0
\le
\frac{1}{2}
\left(
\|\hat\theta_t-\theta^*\|_{\V_{t+1}}^2
-
\|\hat\theta_{t+1}-\theta^*\|_{\V_{t+1}}^2
\right)
-
\frac{c_\mu}{2}
\big(x_t^\top\hat\theta_t-x_t^\top\theta^*\big)^2
+
\eta_t
\big(x_t^\top\hat\theta_t-x_t^\top\theta^*\big)
+
\frac{1}{2}
\|\nabla\ell_{t,i}(\hat\theta_t)\|_{\V_{t+1}^{-1}}^2 .
\end{aligned}
\]
Since
$
\V_{t+1}
=
\V_t+\frac{c_\mu}{2}x_tx_t^\top,
$
we have
\[
\|\hat\theta_t-\theta^*\|_{\V_{t+1}}^2
=
\|\hat\theta_t-\theta^*\|_{\V_t}^2
+
\frac{c_\mu}{2}
\big(x_t^\top\hat\theta_t-x_t^\top\theta^*\big)^2 .
\]
Substituting this identity into the previous inequality yields
\[
\begin{aligned}
0
\le
\frac{1}{2}
\left(
\|\hat\theta_t-\theta^*\|_{\V_t}^2
-
\|\hat\theta_{t+1}-\theta^*\|_{\V_{t+1}}^2
\right)
-
\frac{c_\mu}{4}
\big(x_t^\top\hat\theta_t-x_t^\top\theta^*\big)^2
+
\eta_t
\big(x_t^\top\hat\theta_t-x_t^\top\theta^*\big)
+
\frac{1}{2}
\|\nabla\ell_{t,i}(\hat\theta_t)\|_{\V_{t+1}^{-1}}^2 .
\end{aligned}
\]
By Assumption~\ref{ass:bounded_noise}, $|\eta_t|\le R$ and $|\mu_i(\cdot)|\le U$. Hence $|y_{t,i}|\le U+R$ and $|\mu_i(x_t^\top\hat\theta_t)-y_{t,i}|\le 2(U+R).$ Using a slightly loose bound, we have
\[
\|\nabla\ell_{t,i}(\hat\theta_t)\|_{\V_{t+1}^{-1}}^2
\le
4(U+R)^2\|x_t\|_{\V_{t+1}^{-1}}^2.
\]
Therefore,
\[
\begin{aligned}
0
\le
\frac{1}{2}
\left(
\|\hat\theta_t-\theta^*\|_{\V_t}^2
-
\|\hat\theta_{t+1}-\theta^*\|_{\V_{t+1}}^2
\right)
-
\frac{c_\mu}{4}
\big(x_t^\top\hat\theta_t-x_t^\top\theta^*\big)^2
+
\eta_t
\big(x_t^\top\hat\theta_t-x_t^\top\theta^*\big)
+
2(U+R)^2\|x_t\|_{\V_{t+1}^{-1}}^2 .
\end{aligned}
\]
Multiplying by two and rearranging proves the stated recursion.
$\hfill\square$

\begin{lem}[Uniform control of the noise cross term]
\label{lem:noise_cross_term}
Fix an objective $i\in[m]$ and let
$a_t=f_i(\X_t)^\top(\hat\theta_{t,i}-\theta_i^*)$. Under
Assumptions~\ref{ass:bounded_norm}--\ref{ass:bounded_noise}, with
probability at least $1-\delta/m$, simultaneously for all $t\geq1$,
\[
2\sum_{\tau=1}^t\eta_{\tau,i}a_\tau
\leq
\frac{c_\mu}{2}
+
\frac{c_\mu}{2}\sum_{\tau=1}^t a_\tau^2
+
\frac{16R^2}{c_\mu}
\log\left(
\frac{m\sqrt{1+4S^2t}}{\delta}
\right).
\]
\end{lem}

\paragraph{Proof.}
The sequence $\{\eta_{t,i}a_t\}_{t\geq1}$ is a martingale difference
sequence and $|\eta_{t,i}|\leq R$. The self-normalized martingale
inequality \citep{Abbasi:2012} therefore gives, simultaneously for all
$t\geq1$,
\[
\sum_{\tau=1}^t\eta_{\tau,i}a_\tau
\leq
R\sqrt{
2\left(1+\sum_{\tau=1}^t a_\tau^2\right)
\log\left(
\frac{m\sqrt{1+\sum_{\tau=1}^t a_\tau^2}}{\delta}
\right)
}.
\]
Since $\|f_i(\X_\tau)\|_2\leq1$ and both
$\|\hat\theta_{\tau,i}\|_2$ and $\|\theta_i^*\|_2$ are at most $S$,
we have $|a_\tau|\leq2S$. Thus the logarithm is at most
\[
\log\left(\frac{m\sqrt{1+4S^2t}}{\delta}\right).
\]
Applying Young's inequality to the resulting square-root term gives the
claimed bound.
$\hfill\square$

\begin{lem}\label{lem:online_estimation_error}
Suppose Assumptions~\ref{ass:bounded_norm}--\ref{ass:bounded_noise} hold.
With probability at least $1-\delta$, for all $i\in[m]$ and all $t\ge1$,
\[
\begin{aligned}
\|\hat{\theta}_{t+1,i}-\theta_i^*\|_{\V_{t+1,i}}^2
\le
\|\theta_i^*\|_{\Lambda}^2
+
\frac{16(U+R)^2}{c_\mu}
\log\frac{\det(\V_{t+1,i})}{\det(\Lambda)}
+
\frac{c_\mu}{2}
+
\frac{16R^2}{c_\mu}
\log\left(
\frac{m\sqrt{1+4S^2t}}{\delta}
\right).
\end{aligned}
\]
\end{lem}

\paragraph{Proof.}
Fix $i\in[m]$, suppress the objective index, and let
$a_t=x_t^\top(\hat\theta_t-\theta^*)$. Summing
Lemma~\ref{lem:one_step_estimation} over $\tau=1,\ldots,t$, and using
$\hat\theta_1=0$ and $\V_1=\Lambda$, gives
\[
\begin{aligned}
\|\hat\theta_{t+1}-\theta^*\|_{\V_{t+1}}^2
\leq
\|\theta^*\|_\Lambda^2
-\frac{c_\mu}{2}\sum_{\tau=1}^t a_\tau^2
+4(U+R)^2\sum_{\tau=1}^t
\|x_\tau\|_{\V_{\tau+1}^{-1}}^2
+2\sum_{\tau=1}^t\eta_\tau a_\tau .
\end{aligned}
\]
The standard elliptical-potential argument \citep{Hazan:2007} yields
\[
\sum_{\tau=1}^t\|x_\tau\|_{\V_{\tau+1}^{-1}}^2
\leq
\frac{4}{c_\mu}
\log\frac{\det(\V_{t+1})}{\det(\Lambda)}.
\]
On the event in Lemma~\ref{lem:noise_cross_term}, its quadratic term
exactly cancels the negative strong-convexity term above. Substitution
therefore proves the claimed bound for objective $i$. A union bound over
$i\in[m]$ completes the proof.
$\hfill\square$

By Lemma~C.5 of \citet{Kang:2022}, together with
$\lambda_0\geq c_\mu/2$ and $\log(1+x)\leq x$, for every objective
$i\in[m]$,
\begin{equation*}
\begin{aligned}
\log\frac{\text{det}(\V_{t+1})}{\text{det}(\Lambda)}
\le k\log\left(1+\frac{t}{k}\right)+\frac{c_\mu t}{2\lambda_\perp}.
\end{aligned}
\end{equation*}
Substituting this bound into Lemma~\ref{lem:online_estimation_error}, we obtain that for all $i\in[m]$,
\begin{equation*}
\begin{aligned}
\norm{\hat{\theta}_{t+1,i}-\theta^*_i}_{\V_{t+1,i}}^2\leq\lambda_0S^2+\lambda_\perp S_\perp^2+\frac{16(U+R)^2}{c_\mu}\cdot\left(k\log\left(1+\frac{t}{k}\right)+\frac{c_\mu t}{2\lambda_\perp}\right)+\frac{c_\mu}{2}+\frac{16R^2}{c_\mu}\log\left(\frac{m\sqrt{1+4S^2t}}{\delta}\right).
\end{aligned}
\end{equation*}
Based on this, we provide the following lemma.

\begin{lem}\label{lem:reward_confidence}
Suppose Assumptions~\ref{ass:bounded_norm}--\ref{ass:bounded_noise} hold. Then, with probability at least $1-\delta$, for all $t\ge 1$, all objectives $i\in[m]$, and all arms $\X\in\SX$, we have
\[
\left|\mu_i\!\left(f_i(\X)^\top\theta_i^*\right)-\mu_i\!\left(f_i(\X)^\top\hat{\theta}_{t,i}\right)\right|
\le\beta_t\|f_i(\X)\|_{\V_{t,i}^{-1}},
\]
where
\[
\begin{aligned}
\beta_{t+1}
=
L_\mu
\Bigg(
\frac{4(U+R)}{\sqrt{c_\mu}}\sqrt{
k\log\left(1+\frac{t}{k}\right)
+
\frac{c_\mu t}{2\lambda_\perp}
+
\log\left(
\frac{m\sqrt{1+4S^2t}}{\delta}
\right)}
+
\sqrt{\frac{c_\mu}{2}}
+
\sqrt{\lambda_0}S+\sqrt{\lambda_\perp}S_\perp
\Bigg).
\end{aligned}
\]
\end{lem}
\paragraph{Proof.} Fix any round $t\geq 1$, objective $i\in[m]$, and arm $\X\in\SX$. Since $\mu_i(\cdot)$ is $L_\mu$-Lipschitz continuous, we have
\begin{equation*}
\begin{aligned}
&\left|
\mu_i\!\left(f_i(\X)^\top\theta_i^*\right)
-
\mu_i\!\left(f_i(\X)^\top\hat{\theta}_{t,i}\right)
\right| 
\leq
L_\mu
\left|
f_i(\X)^\top(\theta_i^*-\hat{\theta}_{t,i})
\right|.
\end{aligned}
\end{equation*}
By the Cauchy--Schwarz inequality under the norm induced by $\V_{t,i}$, we further obtain
\begin{equation*}
\left|
f_i(\X)^\top(\theta_i^*-\hat{\theta}_{t,i})
\right|
\leq
\|f_i(\X)\|_{\V_{t,i}^{-1}}
\|\theta_i^*-\hat{\theta}_{t,i}\|_{\V_{t,i}}.
\end{equation*}
Combining the above inequality gives
\begin{equation*}
\left|
\mu_i\!\left(f_i(\X)^\top\theta_i^*\right)
-
\mu_i\!\left(f_i(\X)^\top\hat{\theta}_{t,i}\right)
\right|
\leq
\beta_t
\|f_i(\X)\|_{\V_{t,i}^{-1}}.
\end{equation*}
Since the confidence event holds uniformly over all $t\geq1$ and $i\in[m]$, the desired result follows.
$\hfill\square$

\paragraph{Step 2: Lexicographic filtering and bounded gaps.}

We next analyze the candidate-set recursion. The first lemma controls every
arm that survives filtering within one round; the second transfers this
control to the arm selected in the following round.

\begin{lem}\label{lem:lex_filtering_gap}
Suppose Assumption~\ref{ass:lex_tradeoff} holds.  Let
$
W_i^{\rm lex}=1+w+\cdots+w^{i-1}.
$
With probability at least $1-\delta$, if $X_*\in\SX_t^0$, then, for every $i\in[m]$, the following two statements hold:
\[
X_*\in\SX_t^i,
\]
and, for every $\X\in\SX_t^i$,
\[
\mu_i(\langle X_*,\Theta_i^*\rangle)
-
\mu_i(\langle \X,\Theta_i^*\rangle)
\le
4W_i^{\rm lex}c_{t,i_t}(\X_t) .
\]
\end{lem}

\paragraph{Proof.} For simplicity, define
$
\Delta_i(\X)
=
\mu_i(\langle X_*,\Theta_i^*\rangle)
-
\mu_i(\langle \X,\Theta_i^*\rangle).
$
Since $(\X_t,i_t)$ maximizes the confidence width over $\SX_t\times[m]$, for every $\X\in\SX_t$ and every $i\in[m]$,
\[
c_{t,i}(\X)\le c_{t,i_t}(\X_t).
\]
Let
$
W_i:=2+4w+\cdots+4w^{i-1}.
$
Then the filtering rule in Algorithm~\ref{alg:lexi_ucb} can be written as
\[
\SX_t^i
=
\left\{
\X\in\SX_t^{i-1}:
\hat y_{t,i}(\X_{t,i})-\hat y_{t,i}(\X)
\le W_ic_{t,i_t}(\X_t)
\right\}.
\]
Note that $W_i+2=4(1+w+\cdots+w^{i-1})=4W_i^{\rm lex}$. We prove the result by induction over the objective index $i$.

For $i=1$, since $X_*$ is lexicographically optimal, it maximizes the first objective. Hence, for any $\X\in\SX_t$,
\[
\mu_1(\langle X_*,\Theta_1^*\rangle)
\ge
\mu_1(\langle \X,\Theta_1^*\rangle).
\]
In particular,
\[
\mu_1(\langle X_*,\Theta_1^*\rangle)
\ge
\mu_1(\langle \X_{t,1},\Theta_1^*\rangle).
\]
Using the confidence event, we obtain
\[
\begin{aligned}
\hat y_{t,1}(\X_{t,1})-\hat y_{t,1}(X_*)
&\le
\mu_1(\langle \X_{t,1},\Theta_1^*\rangle)
-
\mu_1(\langle X_*,\Theta_1^*\rangle)
+c_{t,1}(\X_{t,1})+c_{t,1}(X_*)
\\
&\le
2c_{t,i_t}(\X_t).
\end{aligned}
\]
Since $W_1=2$, this implies $X_*\in\SX_t^1$.

Moreover, for any $\X\in\SX_t^1$,
\[
\begin{aligned}
\Delta_1(\X)
&=
\mu_1(\langle X_*,\Theta_1^*\rangle)
-
\mu_1(\langle \X,\Theta_1^*\rangle)
\\
&\le
\hat y_{t,1}(X_*)-\hat y_{t,1}(\X)+2c_{t,i_t}(\X_t)
\\
&\le
\hat y_{t,1}(\X_{t,1})-\hat y_{t,1}(\X)+2c_{t,i_t}(\X_t)
\\
&\le
W_1c_{t,i_t}(\X_t)+2c_{t,i_t}(\X_t)
\\
&=
4c_{t,i_t}(\X_t)
=
4W_1^{\rm lex}c_{t,i_t}(\X_t).
\end{aligned}
\]
Thus the claim holds for $i=1$.

Now suppose the claim holds for all objectives $j<i$. 
Since the candidate sets are nested,
\[
\SX_t^{i-1}\subseteq \SX_t^j,
\qquad \forall j<i.
\]
Thus, for every $\X\in\SX_t^{i-1}$ and every $j<i$,
\[
\Delta_j(\X)
\le
4W_j^{\rm lex}c_{t,i_t}(\X_t)
\le
4W_{i-1}^{\rm lex}c_{t,i_t}(\X_t).
\]

We first show that $X_*\in\SX_t^i$. Since $\X_{t,i}\in\SX_t^{i-1}$, Assumption~\ref{ass:lex_tradeoff} gives
\[
\begin{aligned}
\mu_i(\langle \X_{t,i},\Theta_i^*\rangle)
-
\mu_i(\langle X_*,\Theta_i^*\rangle)
&\le
w\max_{j\in[i-1]}
\left\{
\mu_j(\langle X_*,\Theta_j^*\rangle)
-
\mu_j(\langle \X_{t,i},\Theta_j^*\rangle)
\right\}
\\
&=
w\max_{j\in[i-1]}\Delta_j(\X_{t,i})
\le
4wW_{i-1}^{\rm lex}c_{t,i_t}(\X_t).
\end{aligned}
\]
Therefore, by the confidence event,
\[
\begin{aligned}
\hat y_{t,i}(\X_{t,i})-\hat y_{t,i}(X_*)
&\le
\mu_i(\langle \X_{t,i},\Theta_i^*\rangle)
-
\mu_i(\langle X_*,\Theta_i^*\rangle)
+c_{t,i}(\X_{t,i})+c_{t,i}(X_*)
\\
&\le
4wW_{i-1}^{\rm lex}c_{t,i_t}(\X_t)+2c_{t,i_t}(\X_t)
\\
&=
W_ic_{t,i_t}(\X_t).
\end{aligned}
\]
Hence $X_*\in\SX_t^i$.

Next, for any $\X\in\SX_t^i$, using the confidence event again gives
\[
\begin{aligned}
\Delta_i(\X)
&=
\mu_i(\langle X_*,\Theta_i^*\rangle)
-
\mu_i(\langle \X,\Theta_i^*\rangle)
\\
&\le
\hat y_{t,i}(X_*)-\hat y_{t,i}(\X)+2c_{t,i_t}(\X_t)
\\
&\le
\hat y_{t,i}(\X_{t,i})-\hat y_{t,i}(\X)+2c_{t,i_t}(\X_t)
\\
&\le
W_ic_{t,i_t}(\X_t)+2c_{t,i_t}(\X_t)
\\
&=
4W_i^{\rm lex}c_{t,i_t}(\X_t).
\end{aligned}
\]
This completes the induction and the proof.
$\hfill\square$

\begin{lem}\label{lem:bounded_gap_lex}
Suppose Assumption~\ref{ass:lex_tradeoff} holds. Let
$
W_i^{\rm lex}=1+w+\cdots+w^{i-1}.
$
With probability at least $1-\delta$, for every decision round $t\geq 1$,
the lexicographically optimal arm satisfies
\[
\X_*\in\SX_t.
\]
Moreover, for every $t\geq 2$ and every objective $i\in[m]$, the arm $\X_t$
selected at round $t$ satisfies
\[
\mu_i\!\left(\langle \X_*,\Theta_i^*\rangle\right)
-
\mu_i\!\left(\langle \X_t,\Theta_i^*\rangle\right)
\le
4W_i^{\rm lex}c_{t-1,i_{t-1}}(\X_{t-1}).
\]
\end{lem}

\paragraph{Proof.}
We first prove by induction on $t$ that $\X_*\in\SX_t$. By initialization,
$\SX_1=\SX$, and hence $\X_*\in\SX_1$. Suppose that $\X_*\in\SX_t$ for
some $t\geq1$. Since $\SX_t^0=\SX_t$, Lemma~\ref{lem:lex_filtering_gap}
implies that $\X_*\in\SX_t^m$. By the candidate-set update
$\SX_{t+1}=\SX_t^m$, we obtain $\X_*\in\SX_{t+1}$. Thus, $\X_*$ belongs
to $\SX_t$ for every $t\geq1$.

Next, for every $t\geq2$, Algorithm~\ref{alg:lexi_ucb} selects $\X_t$
from $\SX_t$. Since $\SX_t=\SX_{t-1}^m$, we have
$\X_t\in\SX_{t-1}^m$. Applying Lemma~\ref{lem:lex_filtering_gap} at
round $t-1$ gives, for every $i\in[m]$,
\[
\mu_i\!\left(\langle \X_*,\Theta_i^*\rangle\right)
-
\mu_i\!\left(\langle \X_t,\Theta_i^*\rangle\right)
\le
4W_i^{\rm lex}c_{t-1,i_{t-1}}(\X_{t-1}).
\]
This completes the proof.
$\hfill\square$



\paragraph{Step 3: Regret accumulation and objective-wise potentials.}

We now combine the regret incurred during the initial subspace exploration phase with that accumulated during the subsequent online decision phase. Since the reward of each objective is bounded in absolute value by $U$, the first $T_1$ rounds contribute at most $2UT_1$ to the regret. The first round of the online decision phase contributes at most $2U$, while the remaining $T_2-1$ rounds, where $T_2=T-T_1$, are controlled by Lemma~\ref{lem:bounded_gap_lex}. Therefore,
\begin{align}
R_i(T)&=\sum_{t=1}^T
\left(
  \mu_i\!\left(\left\langle \X_*,\Theta_i^*\right\rangle\right)-\mu_i\!\left(\left\langle \X_t,\Theta_i^*\right\rangle\right)
\right)\nonumber\\
&\leq2UT_1+2U+\sum_{t=2}^{T_2}4W_i^{\rm lex}c_{t-1,i_{t-1}}(\X_{t-1})\nonumber\\
&=2UT_1+2U+\sum_{s=1}^{T_2-1}4W_i^{\rm lex}c_{s,i_s}(\X_s)\nonumber\\
&\leq 2UT_1+2U+4W_i^{\rm lex}\beta_T\sqrt{T_2-1}
\sqrt{\sum_{s=1}^{T_2-1}\|f_{i_s}(\X_s)\|_{\V_{s,i_s}^{-1}}^2},
\label{eq:regret_before_potential}
\end{align}
where the last inequality follows from the monotonicity of $\{\beta_t\}_{t\ge1}$ and the Cauchy--Schwarz inequality.

Next, for each fixed objective, the standard elliptical-potential argument
\citep[Lemma~11]{Abbasi:2011} gives
\[
\sum_{t=1}^{T_2}\|f_i(\X_t)\|_{\V_{t,i}^{-1}}^2
\leq
\frac{4}{c_\mu}
\log\frac{\det(\V_{T_2+1,i})}{\det(\Lambda)}.
\]
Applying Lemma~C.5 of \citet{Kang:2022}, using
$\lambda_0\geq c_\mu/2$ and $\log(1+x)\leq x$, therefore yields
\begin{equation}\label{eq:cumulative_elliptical_norm}
\sum_{t=1}^{T_2}\|f_i(\X_t)\|_{\V_{t,i}^{-1}}^2
\leq
\frac{4}{c_\mu}
\left(
k\log\left(1+\frac{T}{k}\right)
+
\frac{c_\mu T}{2\lambda_\perp}
\right).
\end{equation}
Because the selected objective $i_s$ may vary across rounds, we partition the
elliptical norms according to the selected objective. Applying
\eqref{eq:cumulative_elliptical_norm} separately to each objective gives
\begin{equation}\label{eq:selected_objective_elliptical_norm}
\begin{aligned}
\sum_{s=1}^{T_2-1}
\|f_{i_s}(\X_s)\|_{\V_{s,i_s}^{-1}}^2
&=
\sum_{j=1}^m
\sum_{\substack{s\in[T_2-1]\\i_s=j}}
\|f_j(\X_s)\|_{\V_{s,j}^{-1}}^2
\\
&\le
\sum_{j=1}^m
\sum_{s=1}^{T_2-1}
\|f_j(\X_s)\|_{\V_{s,j}^{-1}}^2
\\
&\le
\frac{4m}{c_\mu}
\left(
k\log\left(1+\frac{T}{k}\right)
+
\frac{c_\mu T}{2\lambda_\perp}
\right).
\end{aligned}
\end{equation}
Substituting \eqref{eq:selected_objective_elliptical_norm} into
\eqref{eq:regret_before_potential} yields
\begin{equation*}
\begin{aligned}
R_i(T)&\leq 2UT_1+2U+8W_i^{\rm lex}\beta_T\sqrt{m(T_2-1)}
\sqrt{\left(\frac{k}{c_\mu}\log\left(1+\frac{T}{k}\right)+\frac{T}{2\lambda_\perp}\right)}
\end{aligned}
\end{equation*}

\paragraph{Step 4: Parameter choice and completion of the proof.}

We now substitute the choices of $T_1$, $\lambda_\perp$, and $S_\perp$ into the regret bound for \textsc{Lexi-LowGLM}. 
Let
\[
A
:=
\frac{M(d_1+d_2)r}{D_{rr}^2}
\log\left(\frac{m(d_1+d_2)}{\delta}\right).
\]
Then the prescribed exploration length and the tail-coordinate bound satisfy
\[
T_1\asymp \frac{\sqrt{M(d_1+d_2)rT\log((d_1+d_2)m/\delta)}}{D_{rr}} =\sqrt{AT},
\quad
S_\perp
=
\frac{(d_1+d_2)Mr}{T_1D_{rr}^2}
\log\left(\frac{m(d_1+d_2)}{\delta}\right)
\asymp
\sqrt{\frac{A}{T}}.
\]

Recall that
\[
\lambda_\perp
=
\frac{c_\mu T}{k\log\!\left(1+c_\mu T/(k\lambda_0)\right)}.
\]
Hence,
\[
\frac{c_\mu T}{2\lambda_\perp}
=
\frac{k}{2}
\log\!\left(1+\frac{c_\mu T}{k\lambda_0}\right)
\text{ and }
\frac{T}{2\lambda_\perp}
=
\frac{k}{2c_\mu}
\log\!\left(1+\frac{c_\mu T}{k\lambda_0}\right).
\]
Therefore,
\[
\sqrt{
\frac{k}{c_\mu}\log\left(1+\frac{T}{k}\right)
+
\frac{T}{2\lambda_\perp}
}
=
\widetilde O\left(\sqrt{\frac{k}{c_\mu}}\right).
\]

Next, by the definition of $\beta_{T+1}$,
\[
\begin{aligned}
\beta_{T+1}
=
L_\mu
\Bigg(
\frac{4(U+R)}{\sqrt{c_\mu}}
\sqrt{
k\log\left(1+\frac{T}{k}\right)
+
\frac{c_\mu T}{2\lambda_\perp}
+
\log\left(
\frac{m\sqrt{1+4S^2T}}{\delta}
\right)}
+
\sqrt{\frac{c_\mu}{2}}
+
\sqrt{\lambda_0}S+\sqrt{\lambda_\perp}S_\perp
\Bigg).
\end{aligned}
\]
Using the above expression of $\lambda_\perp$, we obtain
\[
\sqrt{
k\log\left(1+\frac{T}{k}\right)
+
\frac{c_\mu T}{2\lambda_\perp}
+
\log\left(
\frac{m\sqrt{1+4S^2T}}{\delta}
\right)}
=
\widetilde O(\sqrt{k}).
\]
Furthermore,
\[
\sqrt{\lambda_\perp}S_\perp
=
\sqrt{
\frac{c_\mu T}{
k\log(1+c_\mu T/(k\lambda_0))}
}
\cdot
\sqrt{\frac{A}{T}}
=
\widetilde O\left(
\sqrt{\frac{c_\mu A}{k}}
\right).
\]
Thus,
\[
\beta_{T+1}
=
\widetilde O\left(
L_\mu
\left(
\frac{(U+R)\sqrt{k}}{\sqrt{c_\mu}}
+
\sqrt{c_\mu}
+
\sqrt{\lambda_0}S
+
\sqrt{\frac{c_\mu A}{k}}
\right)
\right).
\]

Since $\beta_t$ is nondecreasing in $t$, we use
$\beta_T\le \beta_{T+1}$.
Combining the above estimates and using
$\sqrt{T_2-1}\leq\sqrt{T_2}\leq\sqrt{T}$, we get
\[
\begin{aligned}
R_i(T)
&\le
2UT_1
+
2U
+
8W_i^{\rm lex}\beta_T\sqrt{mT_2}
\sqrt{
\frac{k}{c_\mu}\log\left(1+\frac{T}{k}\right)
+
\frac{T}{2\lambda_\perp}
}
\\
&=
\widetilde O\left(
U\sqrt{AT}
+
W_i^{\rm lex}\sqrt{m}L_\mu
\left(
\frac{(U+R)k}{c_\mu}
+
\sqrt{k}
+
S\sqrt{\frac{\lambda_0 k}{c_\mu}}
+
\sqrt{A}
\right)
\sqrt{T}
\right).
\end{aligned}
\]
Substituting the definition of $A$, we obtain
\[
\begin{aligned}
R_i(T)
=
\widetilde O\Bigg(
\Bigg[
&U\frac{\sqrt{M(d_1+d_2)r}}{D_{rr}}
+
W_i^{\rm lex}\sqrt{m}L_\mu
\left(
\frac{(U+R)k}{c_\mu}
+
\sqrt{k}
+
S\sqrt{\frac{\lambda_0 k}{c_\mu}}
+
\frac{\sqrt{M(d_1+d_2)r}}{D_{rr}}
\right)
\Bigg]
\sqrt{T}
\Bigg).
\end{aligned}
\]
When $U,R,L_\mu,c_\mu,\lambda_0$, and $S$ are treated as constants, and since $W_i^{\rm lex}\ge1$, this simplifies to
\[
R_i(T)
=
\widetilde O\left(
W_i^{\rm lex}\sqrt{m}
\left(
k+
\frac{\sqrt{M(d_1+d_2)r}}{D_{rr}}
\right)
\sqrt{T}
\right).
\]
Finally, since
$
k=(d_1+d_2)r-r^2\asymp (d_1+d_2)r
$
and
$
\frac{\sqrt{M(d_1+d_2)r}}{D_{rr}}
\lesssim
(d_1+d_2)r,
$
we obtain
\[
R_i(T)
=
\widetilde O\left(
W_i^{\rm lex}\sqrt{m}\,(d_1+d_2)r\sqrt{T}
\right).
\]
The proof of Theorem~\ref{thm:lexi_lowglm_ucb} is finished.
$\hfill\square$

\end{document}